\documentclass[sigconf]{acmart}

\usepackage{amsmath,amssymb,amsfonts}
\usepackage{algorithm}
\usepackage{algorithmic}
\usepackage{booktabs}
\usepackage{multirow}
\usepackage{graphicx}
\usepackage{colortbl}
\usepackage{makecell}
\usepackage{indentfirst}
\definecolor{FedMGSRow}{gray}{0.94}
\newcommand{\bestcell}[1]{\begingroup\boldmath\textbf{#1}\endgroup}
\newcommand{\secondcell}[1]{\underline{#1}}
\newcommand{\accgain}[2]{\makecell{\bestcell{#1}\\[-1pt]{\scriptsize #2}}}

\acmYear{2026}
\acmConference[CIKM '26]{CIKM '26}{November 7--11, 2026}{Rome, Italy}
\acmBooktitle{Proceedings of CIKM '26, October 7--11, 2026, Rome, Italy}
\setcopyright{none}
\acmDOI{}
\acmISBN{}

\newcommand{\ourmethod}{FedMGS}

\begin{document}

\title[Towards Modality-imbalanced Federated Graph Learning: A Data Synthesis-based Approach]{\texorpdfstring{Towards Modality-imbalanced Federated\linebreak[4] Graph Learning: A Data Synthesis-based Approach}{Towards Modality-imbalanced Federated Graph Learning: A Data Synthesis-based Approach}}

\author{Zhengyu Wu, Hongchao Qin, Xunkai Li, Zekai Chen, Rong-Hua Li, Guoren Wang}
\renewcommand{\shortauthors}{Anonymous Author(s)}

\begin{abstract}
Recently, with the emergency of multimodal data, the graph ML community has undergone a profound data-centric paradigm shift. However, such multimodal graph data are often collected by disparate platforms, departments, or institutions, making centralized storage and processing prohibitive due to privacy regulations and potential competitive interests. Although MultiModal Federated Graph Learning (MM-FGL) offers a natural collaborative training paradigm, its practical deployment is challenged by two granularities of modality imbalance. Client-level imbalance occurs when certain clients lack entire modalities, while node-level imbalance occurs when individual nodes exhibit missing visual or textual attributes. While several relevant studies exist, our investigation reveals that they predominantly target graph-agnostic or centralized scenarios, rendering them difficult to adapt directly. To address these challenges, we formalize modality-imbalanced MM-FGL as an implicit graph-aware latent semantic representation synthesis problem. This paradigm recovers missing modal semantics directly within the representation space, thereby maximizing alignment with the original data's semantic distribution and mitigating the high variance induced by missing modalities. To this end, we propose FedMGS (Federated Modality-aware Graph Synthesis), which integrates three core components. The availability-aware graph encoder prevents missing modalities from contaminating local structural propagation. The prototype-guided latent semantic synthesizer establishes cross-client semantic anchors for unavailable modalities. The reliability-calibrated semantic fusion mechanism regulates the impact of recovered latent representations prior to predictive readout. Extensive experiments on four tasks show that FedMGS consistently outperforms competitive baselines with gains up to $17.41\%$ with best efficiency-performance tradeoff.

\end{abstract}



\maketitle

\section{Introduction}

Graphs, as a powerful data format for modeling complex systems characterized by structured relations, have garnered significant research interest. 
Graph Neural Networks learn node representations by propagating and transforming node attributes over graph neighborhoods~\cite{kipf2017gcn,velickovic2018gat,hamilton2017graphsage,xu2018gin,wu2020gnn_survey1,zhou2022gnn_survey2}. 
Their practical value has been demonstrated in successful implementations across realistic sceanrios, including recommendation system~\cite{Social_recomm,cai2023app_gnn_rec3}, biomedical discovery~\cite{bang2023app_gnn_bio1}, and financial security system~\cite{hyun2023app_gnn_fina2}. 
As these scenarios become multimodal, graph nodes increasingly carry heterogeneous profiles such as text and images, while edges encode contextual relations among entities. 
In practice, such multimodal graph datasets are often distributed across institutions in different geographical locations, which makes centralized storage and processing difficult by facing strict privacy regulations and operational constraints caused by competitive interests~\cite{wufederated2023fedapp_gnn_bio2,pan2022fedapp_gnn_fina1}.

Multimodal Federated Graph Learning (MM-FGL) serves as a natural collaborative paradigm for this distributed landscape, enabling clients to train on local multimodal subgraphs while a central server coordinates global model refinement. Nevertheless, a fundamental challenge persists as modality availability is often heterogeneous across institutions and individual dataholders. 
A client may store textual product descriptions while omitting visual data entirely, or data pipelines may fail or update asynchronously, causing sporadic modality loss at the node level. 
We characterize this phenomenon as \emph{two-granularity modality imbalance}, encompassing both client-level modality absence and node-level attribute sparsity. Although these two cases manifest at different scopes, both compel the model to perform sub-optimal inference under partial multimodal inputs.

Existing modality-related learning paradigms often focus on one aspect of this problem, making them inapplicable in more challenging and evolving realistic scenarios.
Centralized multimodal graph learning can exploit topology and heterogeneous attributes~\cite{wei2019mmgcn,tao2020mgat,zhu2025mmgraph,he2025unigraph2} jointly but assumes unified access to graph structures and modality features. 
Missing-modality multimodal FL methods recover or align incomplete samples under partial modality availability~\cite{feng2023fedmultimodal,che2024fedmvp,nguyen2024fedmac,peng2024fedmm,le2025mfcpl}, but they usually treat samples as independent multimodal entities and do not incorporate graph neighborhoods into the semantic completion process.
Consequently, modality-imbalanced MM-FGL still lacks a mechanism for recovering incomplete modality semantics while accounting for graph propagation and various downstream task adaptation.

The central difficulty is not merely that a modality value is absent, but that missing modality semantics alter representation learning through the federated graph pipeline. Three coupled failure mechanisms are especially consequential. \textbf{(1) Propagation Contamination from Invalid Modality Inputs:} When missing attributes are replaced by zero or placeholder vectors and then injected into message passing, the graph encoder may aggregate them as valid semantic evidence. The resulting noise can be diffused through local neighborhoods and further affect efficacy of the local training process. 
\textbf{(2) Cross-client Semantic Inaccessibility under Entire-modality Absence:} When a client lacks an entire modality throughout its subgraph, local training contains no observed samples from that modality and therefore cannot estimate its latent semantic distribution from local semantics alone, even if other clients contain useful modality-specific information due to varying topological structures.
\textbf{(3) Uncertain Reliability of Synthesized Latent Semantics:} Even when a missing latent representation can be synthesized, its usefulness is instance-dependent. The available structural, cross-modal, and semantic context may be mutually consistent for some nodes, yet incomplete, noisy, or ambiguous for others.  Therefore, modality-imbalanced MM-FGL calls for representation-space recovery that is graph-aware, cross-client, and uncertainty-calibrated.

To this end, we propose \ourmethod{} (Federated Modality-aware Graph Synthesis), a client-server framework for implicit graph-aware latent semantic synthesis. To address \textbf{Propagation Contamination}, \ourmethod{} first performs availability-aware graph encoding, where modality indicators gate local inputs so that unavailable attributes are excluded before message passing and cannot be propagated as valid semantic evidence. 
To mitigate \textbf{Cross-client Semantic Inaccessibility}, \ourmethod{} introduces prototype-guided latent semantic synthesis, which leverages local graph context together with federated class-modality prototypes acquired through federated collaboration to construct missing latent representations without transmitting raw features, edges, or node-level embeddings. 
To handle \textbf{Reliability Uncertainty}, \ourmethod{} further adopts reliability-calibrated semantic fusion, which preserves observed latent representations and adaptively controls the contribution of synthesized semantics before downstream evaluation. 
Through these three coordinated designs, \ourmethod{} turns modality imbalance from a raw missing-feature problem into a graph-aware, cross-client, and uncertainty-calibrated latent recovery process.
Empirically, \ourmethod{} achieves the strongest performance across all four evaluation tasks, with improvements of up to $17.41\%$ over competitive baselines while maintaining the most favorable efficiency-performance tradeoff. Further ablation results verify that the three core modules contribute complementary gains, and robustness analyses show that \ourmethod{} consistently preserves its advantage under severe modality-missing scenarios.

\paragraph{Our contributions.}
Accordingly, this paper makes three contributions that connect the problem formulation, method design, and evaluation.
\begin{enumerate}
    \item \textbf{New Perspective.} We formulate modality-imbalanced MM-FGL as a two-granularity incomplete-semantics problem. Missingness at the client and node levels jointly challenges graph propagation, semantic learning across clients, and prediction reliability.
    \item \textbf{New Method.} We propose \ourmethod{}, an implicit latent synthesis framework that combines availability-aware encoding, federated semantic anchors built only from observed entries, and reliability-calibrated fusion without transmitting raw graph data.
    \item \textbf{Versatile Evaluation.} We evaluate \ourmethod{} with graph-structured and modality-specific readouts, including node classification, link prediction, modality matching, and retrieval. The results report both improvements and boundary cases where synthesis should be conservative.
\end{enumerate}
\section{Related Work}

\subsection{(Multimodal) Federated Graph Learning}

FGL trains graph models over client subgraphs, with the server coordinating training across local clients. The FGL survey~\cite{fu2022fgl_survey_1} and position paper~\cite{zhang2021fgl_survey_2} summarize privacy, topology, and heterogeneity challenges, while FedGraphNN~\cite{he2021fedgraphnn} and FederatedScope-GNN~\cite{WangFedScope_22_fsg} provide benchmark systems. Representative methods address these challenges from different perspectives, as FedSage learns neighbor generators for missing cross-client neighbors~\cite{zhang2021fedsage}, FedGL builds global self-supervision to complement local graph information~\cite{chen2021fedgl}, GCFL groups clients by graph-induced training dynamics~\cite{xie2021gcfl}, FedPub personalizes subgraph federated learning~\cite{baek2022fedpub}, and FedStar separates structural knowledge from task-specific parameters for personalization~\cite{tan2023fedstar}. Server collaboration is also shaped by topology, with FGGP using class-wise prototypes under domain shift~\cite{wan2024fgl_fggp}, FedGTA using smoothing confidence and neighbor moments~\cite{li2024fedgta}, FedGCN studying convergence and communication tradeoffs~\cite{yao2024fgl_fedgcn}, FedTAD distilling topology-aware knowledge~\cite{zhu2024fedtad}, and AdaFGL adapting to topology variation~\cite{li2024adafgl}. These studies make topology a central object in FGL, but most of them still assume that each client's node attributes are directly usable before graph propagation begins. Our setting starts one step earlier because the model must decide how to construct usable semantic inputs when a modality is absent.

Multimodal graph learning studies how topology and heterogeneous node attributes jointly improve graph representation learning. MMGCN designs modality-specific graph convolution for micro-video recommendation~\cite{wei2019mmgcn}, MGAT uses multimodal attention for recommendation~\cite{tao2020mgat}, MMGraph benchmarks multimodal graph learning~\cite{zhu2025mmgraph}, and UniGraph2 learns a unified embedding space for multimodal graphs~\cite{he2025unigraph2}. These methods show that multimodal attributes and topology provide complementary semantic signals, but they rely on centralized access to both graph structures and modality features. When this problem moves into MM-FGL, MM-OpenFGL exposes client heterogeneity in modality, topology, and labels~\cite{li2026mmopenfgl}, and STAGE analyzes semantic drift during federated multimodal graph training~\cite{zhang2026stage}. What remains underdeveloped is a mechanism for graph-conditioned semantic construction when missing modalities occur at both client and node granularities. \ourmethod{} addresses this gap by making local neighborhoods and federated class-modality semantics part of the recovery condition itself.

\subsection{Modality-imbalanced Federated Learning}

Multimodal FL has studied clients with incomplete modality sets, and this line of work confirms that semantic recovery is often necessary under modality imbalance. FedMultimodal formalizes missing modalities as a benchmark challenge~\cite{feng2023fedmultimodal}, FedMVP transfers foundation-model knowledge for explicit modality completion~\cite{che2024fedmvp}, and FedMAC learns cross-modal aggregation and latent imputation under partial modalities~\cite{nguyen2024fedmac}. FedMM studies modality heterogeneity in computational pathology~\cite{peng2024fedmm}, MFCPL introduces cross-modal prototypes for severely missing modalities~\cite{le2025mfcpl}, and PEPSY controls embeddings under heterogeneous missing modalities~\cite{nguyen2025pepsy}. Together, these methods cover different points in the synthesis spectrum, from raw or feature-level completion to implicit representation control.

The graph setting changes the recovery problem because a node with missing semantics in one modality is embedded in a local topology, and its representation can affect neighboring nodes through message passing, node-pair scoring, matching, retrieval, and later federated updates. Treating each node as an independent multimodal sample therefore misses the propagation path through which a poor recovery can influence other predictions. \ourmethod{} adapts the recovery route to MM-FGL with availability-aware graph encoding, federated semantic prototypes for each modality, prototype aggregation from observed entries, and reliability-calibrated latent semantic fusion.
\section{Preliminaries}

\paragraph{MM-FGL Notation.}
We study multimodal federated graph learning in a system with $K$ clients and a central server. 
Throughout the paper, $\mathcal{M}={v,t}$ denotes the set of visual and textual modalities, $\mathcal{Y}={1,\ldots,C}$ denotes label indices or semantic-anchor indices when available, $\mathbf{x}$ denotes a vector, and $\mathbf{X}$ denotes a matrix.
Client $k$ owns a local multimodal-attributed graph $G_k=(V_k,E_k,\{\mathbf{X}_k^m\}_{m\in\mathcal{M}},\mathbf{Y}_k)$, where $V_k$ and $E_k$ are the local node and edge sets, $\mathbf{X}_k^m\in\mathbb{R}^{|V_k|\times d_m}$ stores modality-$m$ node features, and $\mathbf{Y}_k$ contains supervision when the readout or semantic anchors are available. For a downstream readout $\xi$, client $k$ forms task instances $\mathcal{T}_k^\xi$ from its local graph. We use $a_{k,i}^m\in\{0,1\}$ as the modality availability indicator, where $a_{k,i}^m=1$ means modality $m$ is observed for node $i$ on client $k$, and $a_{k,i}^m=0$ means it is unavailable. We omit the client subscript when the local client is clear, let $m'$ denote the complementary modality of $m$, define $\mathcal{B}_k=\{i:a_{k,i}^v=a_{k,i}^t=1\}$ as the complete-node set, and define $\mathcal{O}_{k,c}^{m}=\{i\in V_k:y_i=c,a_{k,i}^m=1\}$ as observed entries for class or anchor $c$ and modality $m$.

\paragraph{Two-granularity modality imbalance.}
In deployment, modality availability may differ before federated training begins because visual and textual attributes are produced by different collection pipelines, governed by different permissions, or stored under different platform policies. We represent this mismatch with a modality availability indicator $a_{k,i}^m\in\{0,1\}$, where $a_{k,i}^m=1$ means modality $m$ is observed for node $i$ on client $k$, and $a_{k,i}^m=0$ means it is unavailable. This indicator records real modality presence rather than a random training mask. Modality imbalance appears at two granularities, with client-level missingness meaning that a client lacks modality $m$ throughout its local subgraph, i.e., $a_{k,i}^m=0$ for all $i\in V_k$, and node-level missingness meaning that $a_{k,i}^m$ varies across nodes within a client. Both granularities create incomplete modality-specific semantics, but at different locations in the federated graph system.

\paragraph{Implicit latent synthesis and readouts.}
\ourmethod{} synthesizes latent semantic representations $\hat{\mathbf{z}}_i^m$, not raw features $\hat{\mathbf{x}}_i^m$, and does not add artificial nodes, edges, or persistent synthetic samples. The recovered semantic representations for nodes can be consumed by different readouts, including node classification, link prediction, modality matching, and modality retrieval.

\paragraph{Learning objective.}
For a downstream readout $\xi$, the learning process follows three steps. \textbf{Step 1, task-aware local prediction.} Each client $k$ forms task instances $\mathcal{T}_k^\xi$ from its local graph and optimizes the prediction loss specific to node classification, link prediction, modality matching, or modality retrieval. \textbf{Step 2, evidence-tied latent recovery.} The local objective $\mathcal{L}_k^\xi$ also regularizes recovered latent semantics through reconstruction on complete nodes $\mathcal{B}_k$ and alignment with class-modality prototypes $\bar{\mathbf{p}}_c^m$, so synthesis remains tied to observed evidence. \textbf{Step 3, sample-weighted federated optimization.} The server learns shared parameters $\theta$ by minimizing the sample-weighted objective
\begin{equation}
\min_{\theta}\sum_{k=1}^{K}
\frac{|\mathcal{T}_k^\xi|}{\sum_{j=1}^{K}|\mathcal{T}_j^\xi|}
\mathcal{L}_k^\xi(\theta),
\label{eq:federated_objective}
\end{equation}
where $\mathcal{L}_k^\xi$ is evaluated on $G_k$, $\mathbf{X}_k^v$, $\mathbf{X}_k^t$, optional supervision $\mathbf{Y}_k$, and the availability indicators $\mathbf{A}_k=\{a_{k,i}^m:i\in V_k,m\in\mathcal{M}\}$.

\paragraph{Communication constraints.}
Clients exchange model parameters and compact statistics with the server, and \ourmethod{} restricts these statistics to prototype banks $\bar{\mathbf{P}}^v,\bar{\mathbf{P}}^t$ for each modality, observation counts for each class, and scalar spreads between classes. Local features, topology, and embeddings for individual nodes are excluded from this payload.

\section{The \ourmethod{} Method}

\subsection{Overview}

The goal of \ourmethod{} is to make incomplete multimodal graph inputs usable before they reach prediction heads for downstream tasks. Although missingness at the client and node levels differs in where modality semantics disappear, both cases create the same modeling requirement that each node needs a graph-aware semantic representation built from observed modalities and carefully recovered missing semantics.

\ourmethod{} uses one pipeline for both missingness granularities. Availability-aware encoding first gates local inputs so that only observed modality semantics enter graph propagation. Prototype-guided synthesis then constructs missing latents from local graph context and federated class-modality prototypes when local modality input is unavailable. Reliability-calibrated fusion finally treats recovered latents as uncertain evidence and controls their influence on downstream readouts.

\ourmethod{} keeps synthesis implicit by constructing temporary latent representations rather than synthetic nodes, edges, images, texts, or persistent samples. This design preserves the privacy boundary by letting the server only receives model parameters, modality-specific class prototypes, observation counts, and sample counts. It aggregates parameters, refreshes prototype banks from valid observations, computes inter-class prototype spread as a reliability signal, and broadcasts updated statistics for the next round, as summarized by the client-server flow in Figure~\ref{fig:model architecture}.

\begin{figure*}[t]
  \includegraphics[width=0.998\textwidth]{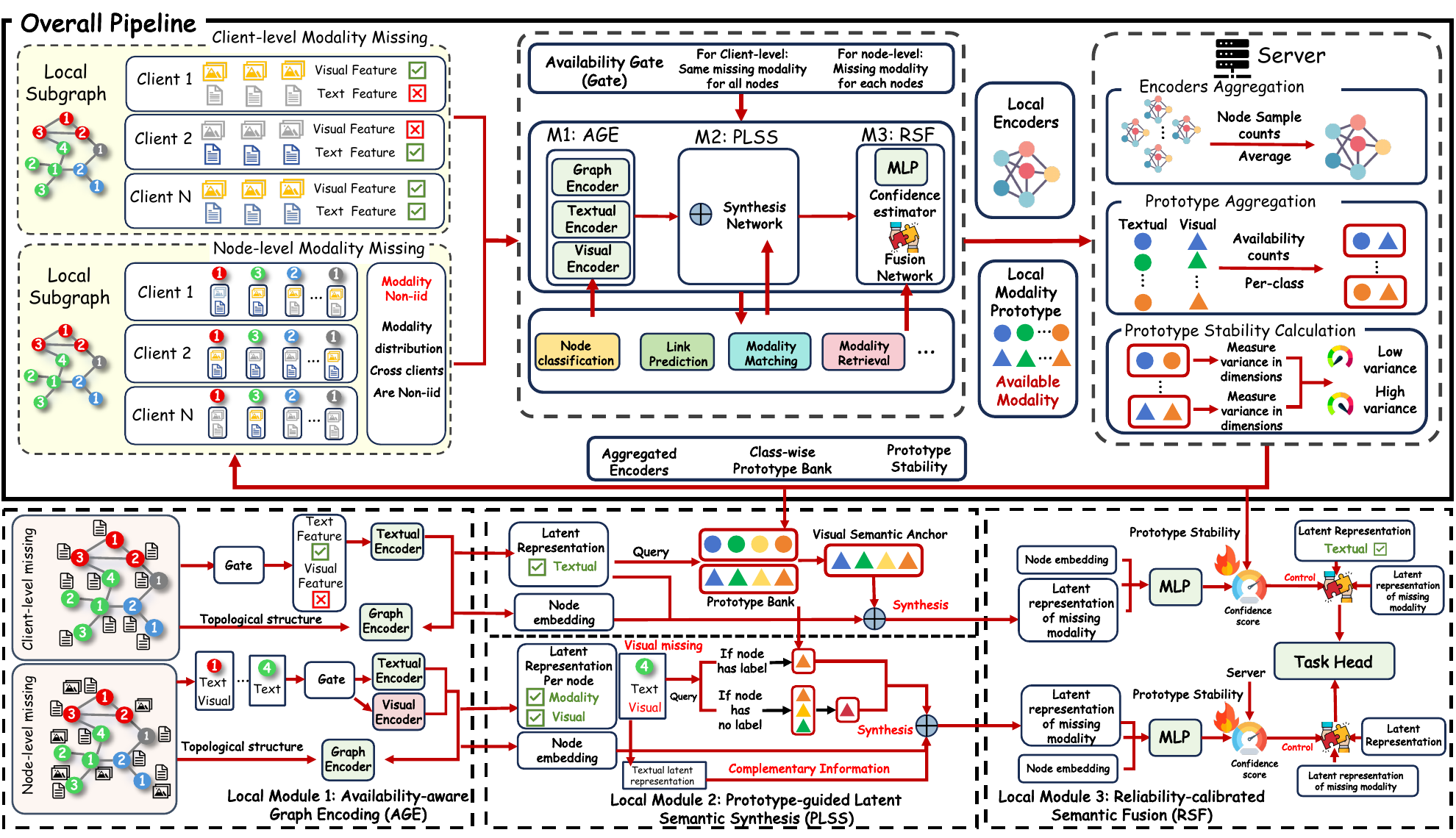}
  \captionsetup{skip=2.5pt, font={small,stretch=1}}
  \caption{Overview of FedMGS. The upper panel shows client-server training under Client-level and Node-level missingness. Each client applies availability gating, graph encoding, prototype-guided latent synthesis, and reliability-calibrated fusion before task learning, while the server aggregates parameters and class-modality prototypes. Lower panels detail the three local modules. Red arrows indicate training flow and flame icons denote learnable reliability scoring.}
\label{fig:model architecture}
\end{figure*}

\subsection{Availability-aware Graph Encoding}
\label{sec:availability-aware-graph-encoding}

To keep unavailable modalities out of local message passing, \ourmethod{} applies the modality availability indicator before propagation and constructs a graph-conditioned initial representation only from observed modality representations. This design mitigates the degraded case in which a missing image or text feature is replaced by a placeholder and then propagated to neighboring nodes as if it were genuine semantic inputs.
\begin{equation}
\begin{alignedat}{2}
\mathbf{z}_i^m
&=a_i^m e_m(\mathbf{x}_i^m),\; m\in\mathcal{M},
\quad&
\mathbf{s}_i
&=\frac{\sum_{m\in\mathcal{M}}\mathbf{z}_i^m}
{\sum_{m\in\mathcal{M}}a_i^m+\epsilon},\\
\mathbf{H}_k^{(0)}
&=[\mathbf{s}_i]_{i\in V_k},
&
\mathbf{H}_k^{(\ell+1)}
&=\operatorname{LN}\big(\operatorname{GNN}^{(\ell)}(\mathbf{H}_k^{(\ell)},E_k)\big).
\end{alignedat}
\label{eq:graph_semantic_seed}
\end{equation}
The graph context for each node is taken from the final propagated state.
\[
\mathbf{h}_i^g=(\mathbf{H}_k^{(L_g)})_i .
\]
Here $\operatorname{LN}$ denotes layer normalization, and $L_g$ is the number of graph propagation layers. The initial semantic representation $\mathbf{s}_i$ excludes unavailable modalities while normalizing representation magnitude by the number of observed modalities. This gating operation adapts to each node's modality pattern under node-level missingness and still extracts graph context from whichever modality remains available under client-level missingness. The final representation $\mathbf{h}_i^g$ is computed locally for each node, so topology can guide later synthesis without transmitting edges or node embeddings to the server.

\subsection{Prototype-guided Latent Semantic Synthesis}
\label{sec:prototype-guided-latent-semantic-synthesis}

To provide modality information that a client cannot derive locally, \ourmethod{} summarizes observed modality representations into class-level prototypes and coordinates these summaries through the server. For each modality $m$, the resulting prototype set provides a compact federated reference for class-specific semantics. 
A client without visual attributes, for example, may still obtain graph context and textual semantics for its nodes, but it has no local visual examples from which to learn a visual latent space. Since the modality-$m$ prototype set is estimated only from observed inputs, $\bar{\mathbf{p}}_c^m$ captures class-$c$ semantics in that modality space rather than collapsing into a generic class summary.

After local training, each client computes a prototype estimate by averaging representation of observed entries for every available class-modality pair, and the server aggregates only valid client-side prototype estimates to form modality-specific prototype sets:
\begin{equation}
\begin{aligned}
n_{k,c}^m
&=|\mathcal{O}_{k,c}^{m}|,\qquad
\mathbf{p}_{k,c}^{m}
=\frac{1}{n_{k,c}^{m}}\sum_{i\in\mathcal{O}_{k,c}^{m}} e_m(\mathbf{x}_i^m),\\
\bar{\mathbf{p}}_{c}^{m}
&=\frac{\sum_{k\in\mathcal{S}^{(t)}}n_{k,c}^{m}\mathbf{p}_{k,c}^{m}}
{\sum_{k\in\mathcal{S}^{(t)}}n_{k,c}^{m}+\epsilon},
\quad n_{k,c}^m>0,\; c\in\mathcal{Y},\;m\in\mathcal{M}.
\end{aligned}
\label{eq:proto_bank}
\end{equation}
This averaging step compresses local modality representations into compact client-side prototype estimates. The accompanying counts allow the server to distinguish unavailable class-modality entries from valid estimates computed from limited observations. Weighting by $n_{k,c}^{m}$ makes each global prototype reflect the amount of observed data supporting each client-side estimate, while excluding missing entries prevents placeholder values from attenuating the prototype for that modality.

The synthesizer constructs a latent representation for a missing modality by combining three sources with distinct roles. The complementary observed modality provides node-specific semantic information, the graph-context representation $\mathbf{h}_i^g$ encodes structural signals aggregated from the local neighborhood, and the prototype query $\boldsymbol{\pi}_i^m$ introduces class-aware modality semantics coordinated across clients. These conditioning terms are concatenated with an alignment placeholder and passed through the synthesis network as follows:
\begin{equation}
\begin{alignedat}{2}
\boldsymbol{\pi}_i^m
&=(\mathbf{q}_i)^\top\bar{\mathbf{P}}^m,
\qquad&
\hat{\mathbf{z}}_i^m
&=\psi_{\mathrm{syn}}^m(\mathbf{u}_i^m),
\\
\mathbf{u}_i^m
&=[\mathbf{z}_i^{m'}\;\|\;\mathbf{h}_i^g\;\|\;\boldsymbol{\pi}_i^m\;\|\;\mathbf{0}],
&\quad m&\in\mathcal{M}.
\end{alignedat}
\label{eq:proto_synthesis}
\end{equation}
For labeled training nodes, $\mathbf{q}_i$ is set to the one-hot encoding of $y_i$. For unlabeled or task-specific nodes, it is given by the current model's class-specific estimate. The prototype query $\boldsymbol{\pi}_i^m=(\mathbf{q}_i)^\top\bar{\mathbf{P}}^m$ therefore selects class-specific semantics from the modality-$m$ prototype set for node $i$. The zero vector in $\mathbf{u}_i^m$ keeps the input dimensions aligned across the two synthesis directions. Since $\hat{\mathbf{z}}_i^m$ is a latent semantic representation rather than a reconstructed raw feature, it can draw on modality-$m$ information aggregated from other clients under client-level missingness, while using the node's observed modality and local graph context to handle node-level missingness.

\subsection{Reliability-calibrated Semantic Fusion}
\label{sec:reliability-calibrated-semantic-fusion}

Since synthesis could not be trusted uniformly, \ourmethod{} treats recovered latents as uncertain semantic estimates rather than direct replacements. A recovered latent is useful when cross-modal semantics, graph context, and prototype-derived semantics are consistent, but it can be misleading when the local context is sparse or the prototype condition is sensitive to class ambiguity.

\ourmethod{} first estimates synthesis confidence for each node from the recovered modality and graph context, while the server tracks prototype stability over valid classes so that the two signals can jointly calibrate recovered representations.
\begin{equation}
\label{eq:synthesis_confidence}
    \eta_i^m
=\sigma\big(\psi_{\mathrm{conf}}^m([\hat{\mathbf{z}}_i^m\;\|\;\mathbf{h}_i^g])\big).
\end{equation}
The server computes stability only over classes with valid prototypes for the target modality.
\begin{equation}
\label{eq:prototype_spread}
\begin{aligned}
\mathcal{Y}_m^{(t)}
&=\{c\in\mathcal{Y}:\sum_{k\in\mathcal{S}^{(t)}}n_{k,c}^{m}>0\},\\
\Delta_m
&=\frac{1}{d}\sum_{r=1}^{d}
\operatorname{Var}\big(\{(\bar{\mathbf{p}}_{c}^{m})_r:c\in\mathcal{Y}_m^{(t)}\}\big).
\end{aligned}
\end{equation}

The node confidence and server-side prototype stability then define both the calibrated confidence and the candidate latent.
\begin{equation}
\label{eq:calibrated_candidate}
\tilde{\eta}_i^m=\eta_i^m\,\sigma(-\beta_m\Delta_m),
\qquad
\tilde{\mathbf{z}}_i^m
=\tilde{\eta}_i^m\hat{\mathbf{z}}_i^m+(1-\tilde{\eta}_i^m)\mathbf{h}_i^g .
\end{equation}
The final fusion rule uses the calibrated candidate only when modality $m$ is missing.
\begin{equation}
\mathbf{f}_i^m
=
\begin{cases}
\mathbf{z}_i^m, & a_i^m=1,\\
\tilde{\mathbf{z}}_i^m, & a_i^m=0,
\end{cases}
\quad m\in\mathcal{M}.
\label{eq:reliability_fusion}
\end{equation}
Here $\sigma(\cdot)$ is the sigmoid function and $\beta_m$ is a learnable scale. The stability $\Delta_m$ turns semantic dispersion between classes into a conservative signal, so the fusion rule relies less on synthesized content when prototype semantics are less stable. Eq.~\ref{eq:reliability_fusion} keeps observed modalities unchanged and, for missing modalities, selects between the synthesized latent and graph context by shifting weight toward local topology when synthesis is uncertain.

The calibrated modality representations are finally combined into node and modality-specific readout representations as follows: 
\begin{equation}
\mathbf{r}_i=\phi_{\mathrm{fuse}}([\mathbf{f}_i^v\;\|\;\mathbf{f}_i^t]),
\qquad
\mathbf{e}_i^m=\phi_m(\mathbf{f}_i^m),\;m\in\mathcal{M}.
\label{eq:readout_representations}
\end{equation}
Downstream readouts consume the calibrated representations according to the prediction target. Node classification predicts $\mathbf{o}_i=g_{\mathrm{cls}}(\mathbf{r}_i)$. Link prediction scores a node pair with $s_{ij}=g_{\mathrm{link}}(\mathbf{r}_i,\mathbf{r}_j)$, following GNN link prediction settings where pairwise structure is inferred from learned node embeddings~\cite{Zhang18link_prediction1}. Link prediction can also be formulated by transforming edges or candidate links into graph entities and applying graph neural networks over their induced relations~\cite{cai2021link_prediction2}. Modality matching compares $\mathbf{e}_i^v$ and $\mathbf{e}_i^t$ for the same node, while modality retrieval ranks candidates across modalities by similarities between $\mathbf{e}_i^v$ and $\mathbf{e}_j^t$. Because missingness is handled before prediction, these heads will not be affected.

\subsection{Federated Optimization}

Each client optimizes the task readout while keeping synthesis tied to observed evidence. Let $\xi$ index the readouts for downstream tasks. For prototype alignment, define $\bar{\mathbf{p}}_c=|\mathcal{M}|^{-1}\sum_{m\in\mathcal{M}}\bar{\mathbf{p}}_c^m$ and $\ell_{i,c}=\operatorname{sim}(\mathbf{s}_i,\bar{\mathbf{p}}_c)/\tau_p$, which lets the local objective combine task supervision, reconstruction on complete nodes, and prototype alignment. Specific formulas are as follows: 
\begin{equation}
\begin{aligned}
\mathcal{L}_{\mathrm{task}}^{k,\xi}
&=\text{readout loss for task }\xi,\\
\mathcal{L}_{\mathrm{rec}}^k
&=\frac{1}{|\mathcal{B}_k||\mathcal{M}|}
\sum_{i\in\mathcal{B}_k}\sum_{m\in\mathcal{M}}
\|\hat{\mathbf{z}}_i^m-\mathbf{z}_i^m\|_2^2,\\
\mathcal{L}_{\mathrm{proto}}^k
&=-\frac{1}{|\mathcal{T}_k^y|}\sum_{i\in\mathcal{T}_k^y}
\log\frac{\exp(\ell_{i,y_i})}{\sum_{c\in\mathcal{Y}}\exp(\ell_{i,c})},\\
\mathcal{L}_k^\xi
&=\mathcal{L}_{\mathrm{task}}^{k,\xi}
+\lambda_{\mathrm{rec}}\mathcal{L}_{\mathrm{rec}}^k
+\lambda_{\mathrm{proto}}\mathcal{L}_{\mathrm{proto}}^k .
\end{aligned}
\label{eq:local_objective}
\end{equation}
where $\operatorname{sim}(\cdot,\cdot)$ is cosine similarity, $\tau_p=\sqrt d$, and $\mathcal{T}_k^y$ denotes labeled training nodes when class supervision is available. For different readouts, the task term is instantiated as supervised cross-entropy loss for node classification, binary or ranking loss for link prediction, contrastive loss for modality matching, or retrieval loss for cross-modal ranking, so the same synthesis backbone can serve different objectives. Because naturally missing modalities do not provide ground-truth targets, reconstruction is restricted to complete nodes in $\mathcal{B}_k$. The prototype term aligns local semantic representations with their corresponding class prototypes and contrasts them against other class prototypes only when class supervision is available. Early rounds rely mainly on task and reconstruction losses, whereas later rounds can exploit more stable prototype sets as cross-client semantic references.

Algorithm~\ref{alg:fedmgs} presents the client-server training procedure of \ourmethod{}. Each round begins with the server distributing the current model parameters together with modality-specific prototype sets and inter-class spread statistics. Using only its local graph and attributes, each selected client performs availability-aware encoding, prototype-conditioned latent synthesis, reliability-calibrated fusion, and task-specific optimization. The client then returns updated parameters and valid class-modality prototype estimates with their observation counts. The server aggregates the received parameters, updates the modality-specific prototype sets from valid estimates, and recomputes spread statistics for the next communication round.

\begin{algorithm}[t]
\caption{\ourmethod{} Training Pipeline}
\label{alg:fedmgs}
\begin{algorithmic}[1]
\STATE \textbf{Input:} clients $\mathcal{K}$, local data $\{G_k,\mathbf{A}_k\}_{k\in\mathcal{K}}$, rounds $T$, local epochs $E$, readout $\xi$, loss weights $\lambda_{\mathrm{rec}},\lambda_{\mathrm{proto}}$
\STATE Initialize $\theta^{(0)}$, prototype banks $\{\bar{\mathbf{P}}^{m,(0)}\}_{m\in\mathcal{M}}$, prototype counts, and spreads $\{\Delta_m^{(0)}\}_{m\in\mathcal{M}}$
\FOR{$t=0$ to $T-1$}
    \STATE Server samples $\mathcal{S}^{(t)}$ and broadcasts $\theta^{(t)}$, $\{\bar{\mathbf{P}}^{m,(t)},\Delta_m^{(t)}\}_{m\in\mathcal{M}}$
    \FOR{each client $k\in\mathcal{S}^{(t)}$ in parallel}
        \STATE Set local parameters $\theta_k\leftarrow\theta^{(t)}$
        \FOR{$e=1$ to $E$}
            \STATE Compute $\mathbf{z}_i^m$, $\mathbf{s}_i$, and $\mathbf{h}_i^g$ from $G_k,\mathbf{A}_k$ by Eq.~\ref{eq:graph_semantic_seed}
            \STATE Build $\mathbf{q}_i$ from labels or semantic anchors when available, and from current predictions otherwise
            \STATE Query $\boldsymbol{\pi}_i^m$ and synthesize $\hat{\mathbf{z}}_i^m$ for each $m\in\mathcal{M}$ by Eq.~\ref{eq:proto_synthesis}
            \STATE Estimate $\eta_i^m$ by Eq.~\ref{eq:synthesis_confidence} and calibrate $\tilde{\eta}_i^m,\tilde{\mathbf{z}}_i^m$ by Eq.~\ref{eq:calibrated_candidate}
            \STATE Fuse $\mathbf{f}_i^m$ by Eq.~\ref{eq:reliability_fusion} and build $\mathbf{r}_i,\mathbf{e}_i^m$ by Eq.~\ref{eq:readout_representations}
            \STATE Update $\theta_k$ using readout $\xi$ and the local objective in Eq.~\ref{eq:local_objective}
        \ENDFOR
        \STATE Compute observed-only $\{\mathbf{p}_{k,c}^m,n_{k,c}^m\}_{c,m}$ by Eq.~\ref{eq:proto_bank}
        \STATE Upload $\theta_k$, $\{\mathbf{p}_{k,c}^m,n_{k,c}^m\}_{c,m}$, and $|\mathcal{T}_k^\xi|$
    \ENDFOR
    \STATE Aggregate $\{\theta_k\}_{k\in\mathcal{S}^{(t)}}$ by sample-weighted averaging following Eq.~\ref{eq:federated_objective}
    \STATE Refresh $\{\bar{\mathbf{P}}^{m,(t+1)}\}_{m\in\mathcal{M}}$ by Eq.~\ref{eq:proto_bank} and recompute $\{\Delta_m^{(t+1)}\}_{m\in\mathcal{M}}$ by Eq.~\ref{eq:prototype_spread}
\ENDFOR
\STATE \textbf{Return:} global model $\theta^{(T)}$ and federated prototype banks $\{\bar{\mathbf{P}}^{m,(T)}\}_{m\in\mathcal{M}}$
\STATE \textbf{Note:} raw features, edges, and node-level embeddings are never uploaded.
\end{algorithmic}
\end{algorithm}

\paragraph{Communication cost.}
Beyond standard model parameters, each client sends modality-specific class prototypes, their observation counts, and one task-sample count per round. This additional payload grows only with the number of modalities, classes, and prototype dimensions, rather than with local nodes or edges, which is important for graph clients with large subgraphs. Communication is therefore limited to model parameters and aggregate prototype statistics, without transmitting raw data.
\section{Theoretical Analysis}

We analyze latent semantic representation construction rather than raw data generation, and the results support three design choices in \ourmethod{}. Specifically, synthesis should be conditioned on graph context and prototype semantics, prototypes should be aggregated only from observed modality entries, and recovered latent representations should be calibrated before they influence the readout. Missingness at the client and node levels differs in the structure of missing semantics for each modality, but both cases induce missing target latents whose recovery can benefit from graph and prototype conditions.

\paragraph{Notation.}
For a fixed target modality $m\in\mathcal{M}$, let $\mathbf{Z}^{m}$ be the target missing-modality latent, $\mathbf{Z}^{m'}$ the complementary modality representation, $\mathbf{G}$ the graph-context representation, and $\boldsymbol{\Pi}^{m}$ the federated prototype query. In \ourmethod{}, these random variables correspond to $\mathbf{z}_i^m$, $\mathbf{z}_i^{m'}$, $\mathbf{h}_i^g$, and $\boldsymbol{\pi}_i^m$ at the node level. For a conditioning set $\mathcal{U}$, define the oracle reconstruction risk under squared loss as
\begin{equation}
R_m(\mathcal{U})=\inf_{\phi}\mathbb{E}\|\mathbf{Z}^{m}-\phi(\mathcal{U})\|_2^2,
\label{eq:oracle_risk}
\end{equation}
where the infimum is over measurable predictors with finite second moment.

\begin{theorem}[Oracle synthesis risk under additional conditions]
Assume $\mathbf{Z}^{m}$ has finite second moment and use the oracle risk in Eq.~\ref{eq:oracle_risk}.
\begin{equation}
R_m(\mathbf{Z}^{m'},\mathbf{G},\boldsymbol{\Pi}^{m})
\le R_m(\mathbf{Z}^{m'}).
\end{equation}
\end{theorem}
\begin{proof}
For squared loss, the optimal predictor given $\mathcal{U}$ is $\mathbb{E}[\mathbf{Z}^{m}\mid\mathcal{U}]$. Because $\mathbf{Z}^{m'}$ is contained in $(\mathbf{Z}^{m'},\mathbf{G},\boldsymbol{\Pi}^{m})$, every predictor using only $\mathbf{Z}^{m'}$ is feasible under the larger conditioning set. Taking the infimum over this larger predictor class gives the claim.
\end{proof}
For \ourmethod{}, this result motivates using graph context and prototype queries as synthesis conditions. Graph context is especially important for missingness at the node level, where available semantics vary across neighboring nodes. Prototype semantics are especially important for missingness at the client level, where a client may lack an entire modality and therefore needs semantic anchors across clients. Because the analysis concerns latent semantic representations for nodes, its implication is independent of whether the readout is node-wise, pairwise, or cross-modal.

\begin{proposition}[Bias of placeholder-based prototype aggregation]
For a fixed class $c$ and modality $m$, let $b_{k,c}^m\in\{0,1\}$ indicate whether client $k$ has a valid observed prototype for $(c,m)$, with $\mathbb{P}(b_{k,c}^m=1)=q_c^m$. Let $\mathcal{I}_{c}^{m}=\{k:b_{k,c}^m=1\}$, let $\mathbf{u}_{k,c}^m$ be the client prototype when $b_{k,c}^m=1$, and assume $\mathbb{E}[\mathbf{u}_{k,c}^m\mid b_{k,c}^m=1]=\boldsymbol{\mu}_c^m$. The two prototype estimators used to compare placeholder-based and observed-only aggregation are
{\small
\begin{equation}
\begin{alignedat}{2}
\widehat{\mathbf{p}}_{\mathrm{plh},c}^{m}
&=K^{-1}\sum_{k=1}^{K} b_{k,c}^m\mathbf{u}_{k,c}^m,
\qquad&
\widehat{\mathbf{p}}_{\mathrm{obs},c}^{m}
&=\frac{1}{|\mathcal{I}_{c}^{m}|}
\sum_{k\in\mathcal{I}_{c}^{m}}\mathbf{u}_{k,c}^m
\end{alignedat}
\end{equation}
}
satisfy $\mathbb{E}[\widehat{\mathbf{p}}_{\mathrm{plh},c}^{m}]=q_c^m\boldsymbol{\mu}_c^m$ and $\mathbb{E}[\widehat{\mathbf{p}}_{\mathrm{obs},c}^{m}\mid |\mathcal{I}_{c}^{m}|>0]=\boldsymbol{\mu}_c^m$.
\end{proposition}
\begin{proof}
For placeholder-based aggregation, each missing client contributes no observed prototype while still being counted in the denominator. Since each summand has mean $q_c^m\boldsymbol{\mu}_c^m$, averaging over $K$ clients gives the first identity and shows that the prototype is attenuated when $q_c^m<1$. Conditioning on the observed set $\mathcal{I}_{c}^{m}$, $\widehat{\mathbf{p}}_{\mathrm{obs},c}^{m}$ is the sample mean of valid observed prototypes, whose conditional expectation is $\boldsymbol{\mu}_c^m$ whenever at least one valid prototype exists.
\end{proof}
\begin{table*}[t]
\caption{Dataset statistics and evaluation settings. The table reports processed graph statistics, evaluation metrics, task types, and the missingness scenario used for each dataset.}
\label{tab:dataset-summary}
\centering
\scriptsize
\setlength{\tabcolsep}{2.8pt}
\renewcommand{\arraystretch}{1.08}
\resizebox{\textwidth}{!}{%
\begin{tabular}{@{}cccccccccc@{}}
\toprule
\textbf{Dataset} & \textbf{Domain} & \textbf{\#Nodes} & \textbf{\#Edges} & \textbf{\#Classes} & \textbf{Modalities} & \textbf{Split} & \textbf{Metric} & \textbf{Missingness Setting} & \textbf{Task} \\
\midrule
Movies & E-Commerce & 16,672 & 148,992 & 20 & Text+Visual & 60/20/20 & Acc. & Client-level & Node Classification \\
Grocery & E-Commerce & 17,074 & 121,724 & 20 & Text+Visual & 60/20/20 & Acc. & Client-level & Node Classification \\
Toys & E-Commerce & 20,695 & 106,058 & 18 & Text+Visual & 60/20/20 & R@5 & Client-level & Modality Retrieval \\
Flickr30k & Image Networks & 31,783 & 181,151 & - & Text+Visual & 60/20/20 & R@5 & Client-level & Modality Retrieval \\
DY & Video Recommendation & 8,299 & 35,627 & - & Text+Visual & 60/20/20 & AUC & Node-level & Link Prediction \\
Bili Dance & Video Recommendation & 2,307 & 9,127 & - & Text+Visual & 60/20/20 & AUC & Node-level & Link Prediction \\
KU & Video Recommendation & 5,370 & 22,052 & - & Text+Visual & 60/20/20 & AUC & Node-level & Modality Match \\
Bili Food & Video Recommendation & 1,579 & 6,544 & - & Text+Visual & 60/20/20 & AUC & Node-level & Modality Match \\
\bottomrule
\end{tabular}
}
\end{table*}
This formalizes modality-isolated aggregation in \ourmethod{}, where prototypes are estimated from observed modality entries instead of artifacts created by missing entries. The distinction is most visible under missingness at the client level, where an entire client can lack one modality and placeholder aggregation would systematically attenuate the corresponding class prototype for that modality.

\begin{theorem}[Reliability weighting with uncertain synthesis]
Let $\hat{\mathbf{Z}}^{m}$ be a synthesized latent semantic representation and $\mathbf{G}$ be a graph-context representation for the same missing latent $\mathbf{Z}^{m}$. Consider $\mathbf{F}_{\alpha}^{m}=\alpha\hat{\mathbf{Z}}^{m}+(1-\alpha)\mathbf{G}$ with $\alpha\in[0,1]$. If both estimators are unbiased for $\mathbf{Z}^{m}$, their errors are uncorrelated, and their mean squared errors are $\nu_{\mathrm{syn}}$ and $\nu_g$, then the risk-minimizing weight on the synthesized representation is
\begin{equation}
\alpha^\star=\frac{\nu_g}{\nu_{\mathrm{syn}}+\nu_g}.
\end{equation}
\end{theorem}
\begin{proof}
Let $\boldsymbol{\epsilon}_{\mathrm{syn}}=\hat{\mathbf{Z}}^{m}-\mathbf{Z}^{m}$ and $\boldsymbol{\epsilon}_g=\mathbf{G}-\mathbf{Z}^{m}$. The corresponding fusion error can be written in the following form.
\begin{equation}
\begin{gathered}
\mathbf{F}_{\alpha}^{m}-\mathbf{Z}^{m}
=\alpha\boldsymbol{\epsilon}_{\mathrm{syn}}+(1-\alpha)\boldsymbol{\epsilon}_g,\\[-1mm]
\mathbb{E}\|\mathbf{F}_{\alpha}^{m}-\mathbf{Z}^{m}\|_2^2
=\alpha^2\nu_{\mathrm{syn}}+(1-\alpha)^2\nu_g ,
\end{gathered}
\end{equation}
where the cross term vanishes by the uncorrelated-error assumption made above. Differentiating with respect to $\alpha$ gives $\alpha\nu_{\mathrm{syn}}-(1-\alpha)\nu_g=0$, hence the stated $\alpha^\star$.
\end{proof}
Under these assumptions, the preferred weight on the synthesized representation decreases as synthesis error increases. Thus, regardless of whether missingness occurs at the client or node level, recovered latent semantic representations should be weighted according to uncertainty before they reach the task-specific readout. In \ourmethod{}, $\Delta_m$ serves as a compact signal of synthesis uncertainty that depends on class semantics during fusion.

\section{Experiments}
\label{sec:experiments}

This section evaluates \ourmethod{} and we organize the experiments around five research questions: \textbf{Q1:} Does \ourmethod{} outperform representative FGL and missing-modality FL baselines across graph and cross-modal tasks? \textbf{Q2:} How stable \ourmethod{} remains as the missing-modality rate changes? \textbf{Q3:} Which core components drive the performance gains.? \textbf{Q4:} How different hyperparameters setting influence the FedMGS performance? \textbf{Q5:} How efficient is FedMGS compared to baselines across graph and cross-modal tasks? The following setup first fixes the datasets, missingness protocol, baseline groups, and implementation details so that the comparisons are interpreted under a common evaluation protocol.

\subsection{Experimental Setup}
\label{sec:experimental-setup}

\paragraph{Datasets.}
We evaluate \ourmethod{} on multimodal graph datasets~\cite{DY_bili_ku,Movies_Grocery_toys,flickr30k} that span four tasks. Table~\ref{tab:dataset-summary} reports the processed graph statistics used in this paper, together with the tasks, metric, and missingness setting associated with each dataset.

\paragraph{Missingness simulation.}
We simulate two-granularity modality imbalance at missing rates $\rho\in\{0.3,0.5,0.7\}$. In Client-level missingness, each selected client loses either the visual or textual modality for all nodes in its local subgraph. 
In Node-level missingness, modality availability varies across nodes within a client to simulate the cross-client modality non-iid scenarios. 

\paragraph{Baselines.}
We compare against baselines chosen to test different sources of advantage. \textbf{FedAvg}~\cite{mcmahan2017fedavg}, \textbf{FedLap}~\cite{aliakbari2025fedlap}, and \textbf{FedGTA}~\cite{li2024fedgta} test whether standard FL or topology-aware FGL training is sufficient under missing modalities. \textbf{FedMAC}~\cite{nguyen2024fedmac}, \textbf{FedMVP}~\cite{che2024fedmvp}, \textbf{MH-pFLID}~\cite{xie2024mhpflid}, and \textbf{PEPSY}~\cite{nguyen2025pepsy} represent missing-modality or multimodal FL adaptations. \textbf{FedProto}~\cite{tan2022fedproto} isolates prototype-based FL without the graph-aware synthesis and reliability fusion used by \ourmethod{}. 
For compact tables, we omit the `-MM' suffix, but all baselines use the same data split, missingness protocol, and task metric.

\paragraph{Implementation details.}
We partition each dataset into either 5 or 10 clients using Louvain community detection~\cite{blondel2008louvain}. 
Training runs for $T=100$ communication rounds with $E=10$ local epochs per client.
We set baselines at their recommended setting or use Adam as the default optimizer.

\paragraph{Experiment environment.}
The experiments are conducted on a machine with an AMD EPYC 7J13 64-Core Processor, and NVIDIA GeForce RTX 4090 with 48GB memory and CUDA 12.6. The operating system is Ubuntu 22.04.5 LTS with 503GB memory.

\paragraph{Data and Code Availability.}
To ensure reproducibility, all datasets used in this study are publicly available. The complete source code for implementing the FedMGS framework is released publicly at \url{https://anonymous.4open.science/r/FedMGS-497C}.

\begin{table*}[t]
\caption{Comparison Test ($\rho=0.5$) for node classification and modality retrieval. Results are reported as mean $\pm$ std. Best results are in \textbf{bold}, and second-best results are underlined.}
\label{tab:comparison-rho05-a}
\centering
\scriptsize
\setlength{\tabcolsep}{1.8pt}
\renewcommand{\arraystretch}{0.9}
\resizebox{0.95\textwidth}{!}{%
\begin{tabular}{@{}lcccccccc@{}}
\toprule
\multirow{3}{*}{Methods} & \multicolumn{4}{c}{Node Classification (Accuracy)} & \multicolumn{4}{c}{Modality Retrieval (R@5)} \\
\cmidrule(lr){2-5}\cmidrule(l){6-9}
 & \multicolumn{2}{c}{Movies} & \multicolumn{2}{c}{Grocery} & \multicolumn{2}{c}{Toys} & \multicolumn{2}{c}{Flickr30k} \\
\cmidrule(lr){2-3}\cmidrule(lr){4-5}\cmidrule(lr){6-7}\cmidrule(l){8-9}
 & Client = 5 & Client = 10 & Client = 5 & Client = 10 & Client = 5 & Client = 10 & Client = 5 & Client = 10 \\
\midrule
FedAvg & $54.03 \pm 0.46$ & $48.80 \pm 0.56$ & $79.96 \pm 0.43$ & $76.61 \pm 0.46$ & $64.22 \pm 0.66$ & $62.46 \pm 0.71$ & $63.61 \pm 0.65$ & $61.32 \pm 0.64$ \\
FedProto & $53.45 \pm 0.46$ & $52.21 \pm 0.63$ & $78.42 \pm 0.36$ & $77.21 \pm 0.49$ & $62.47 \pm 0.67$ & $65.46 \pm 0.72$ & $65.38 \pm 0.60$ & $63.42 \pm 0.63$ \\
FedMAC & $53.12 \pm 0.46$ & $51.56 \pm 0.54$ & $81.29 \pm 0.33$ & \secondcell{$79.99 \pm 0.44$} & $62.83 \pm 0.64$ & $64.56 \pm 0.78$ & $66.32 \pm 0.57$ & $64.21 \pm 0.58$ \\
FedMVP & \secondcell{$55.99 \pm 0.50$} & \secondcell{$54.26 \pm 0.50$} & $79.99 \pm 0.41$ & $78.70 \pm 0.50$ & $65.42 \pm 0.66$ & $66.43 \pm 0.76$ & $65.43 \pm 0.62$ & $63.87 \pm 0.68$ \\
FedLap & $52.40 \pm 0.53$ & $50.85 \pm 0.67$ & $77.41 \pm 0.50$ & $77.30 \pm 0.53$ & $61.15 \pm 0.75$ & $64.25 \pm 0.77$ & $64.42 \pm 0.77$ & $62.54 \pm 0.74$ \\
FedGTA & $55.25 \pm 0.40$ & $52.32 \pm 0.54$ & $80.39 \pm 0.37$ & $78.82 \pm 0.48$ & $66.95 \pm 0.59$ & $68.76 \pm 0.72$ & $67.68 \pm 0.57$ & $64.78 \pm 0.60$ \\
MH-pFLID & $53.76 \pm 0.47$ & $52.13 \pm 0.63$ & $81.21 \pm 0.39$ & $78.64 \pm 0.41$ & $67.76 \pm 0.70$ & $70.21 \pm 0.69$ & \secondcell{$68.43 \pm 0.58$} & \secondcell{$65.32 \pm 0.72$} \\
PEPSY & $51.96 \pm 0.58$ & $50.55 \pm 0.52$ & \secondcell{$81.45 \pm 0.49$} & $79.12 \pm 0.47$ & \secondcell{$68.60 \pm 0.62$} & \secondcell{$70.87 \pm 0.67$} & $66.87 \pm 0.56$ & $64.54 \pm 0.70$ \\
\midrule
\rowcolor{FedMGSRow}
\textbf{\ourmethod{} (Ours)} & \accgain{$65.74 \pm 0.31$}{+17.41\%} & \accgain{$56.31 \pm 0.33$}{+3.78\%} & \accgain{$83.81 \pm 0.24$}{+2.90\%} & \accgain{$82.65 \pm 0.26$}{+3.33\%} & \accgain{$73.30 \pm 0.53$}{+6.85\%} & \accgain{$79.40 \pm 0.51$}{+12.04\%} & \accgain{$72.27 \pm 0.40$}{+5.61\%} & \accgain{$70.79 \pm 0.55$}{+8.37\%} \\
\bottomrule
\end{tabular}
}
\end{table*}
\subsection{Main Results (Answer for Q1)}
\label{sec:main-results}
\begin{table*}[t]
\caption{Comparison Test ($\rho=0.5$) for link prediction and modality match. Results are reported as mean $\pm$ std. Best results are in \textbf{bold}, and second-best results are underlined.}
\label{tab:comparison-rho05-b}
\centering
\scriptsize
\setlength{\tabcolsep}{1.8pt}
\renewcommand{\arraystretch}{0.9}
\resizebox{0.95\textwidth}{!}{%
\begin{tabular}{@{}lcccccccc@{}}
\toprule
\multirow{3}{*}{Methods} & \multicolumn{4}{c}{Link Prediction (AUC)} & \multicolumn{4}{c}{Modality Match (AUC)} \\
\cmidrule(lr){2-5}\cmidrule(l){6-9}
 & \multicolumn{2}{c}{DY} & \multicolumn{2}{c}{Bili Dance} & \multicolumn{2}{c}{KU} & \multicolumn{2}{c}{Bili Food} \\
\cmidrule(lr){2-3}\cmidrule(lr){4-5}\cmidrule(lr){6-7}\cmidrule(l){8-9}
 & Client = 5 & Client = 10 & Client = 5 & Client = 10 & Client = 5 & Client = 10 & Client = 5 & Client = 10 \\
\midrule
FedAvg & $74.79 \pm 0.43$ & $78.34 \pm 0.45$ & $76.16 \pm 0.47$ & $74.92 \pm 0.52$ & $73.31 \pm 0.61$ & $80.56 \pm 0.56$ & $71.61 \pm 0.49$ & $78.57 \pm 0.60$ \\
FedProto & $75.88 \pm 0.46$ & $75.79 \pm 0.44$ & $76.88 \pm 0.58$ & $74.93 \pm 0.55$ & $72.74 \pm 0.59$ & $77.85 \pm 0.67$ & $72.97 \pm 0.54$ & $79.54 \pm 0.59$ \\
FedMAC & $76.75 \pm 0.39$ & $77.70 \pm 0.35$ & $78.45 \pm 0.49$ & \secondcell{$76.74 \pm 0.56$} & $79.53 \pm 0.53$ & $81.15 \pm 0.63$ & $74.54 \pm 0.60$ & $79.35 \pm 0.59$ \\
FedMVP & $77.43 \pm 0.42$ & $78.69 \pm 0.45$ & $77.47 \pm 0.48$ & $75.47 \pm 0.56$ & $75.74 \pm 0.58$ & $79.98 \pm 0.64$ & $73.43 \pm 0.57$ & $76.78 \pm 0.58$ \\
FedLap & $74.98 \pm 0.47$ & $76.57 \pm 0.59$ & $78.48 \pm 0.61$ & $74.36 \pm 0.60$ & $75.22 \pm 0.72$ & $82.49 \pm 0.67$ & $72.63 \pm 0.70$ & $76.54 \pm 0.63$ \\
FedGTA & \secondcell{$78.42 \pm 0.44$} & \secondcell{$78.72 \pm 0.37$} & \secondcell{$79.75 \pm 0.52$} & $76.66 \pm 0.45$ & $78.54 \pm 0.59$ & $80.88 \pm 0.55$ & $73.82 \pm 0.53$ & \secondcell{$80.73 \pm 0.54$} \\
MH-pFLID & $76.03 \pm 0.39$ & $76.56 \pm 0.37$ & $79.32 \pm 0.46$ & $75.53 \pm 0.45$ & $76.87 \pm 0.65$ & $81.43 \pm 0.64$ & $74.34 \pm 0.53$ & $79.54 \pm 0.61$ \\
PEPSY & $76.45 \pm 0.46$ & $77.32 \pm 0.38$ & $78.64 \pm 0.51$ & $76.54 \pm 0.51$ & \secondcell{$79.84 \pm 0.54$} & \secondcell{$82.63 \pm 0.58$} & \secondcell{$75.41 \pm 0.52$} & $79.67 \pm 0.62$ \\
\midrule
\rowcolor{FedMGSRow}
\textbf{\ourmethod{} (Ours)} & \accgain{$79.90 \pm 0.29$}{+1.89\%} & \accgain{$79.82 \pm 0.18$}{+1.40\%} & \accgain{$82.46 \pm 0.37$}{+3.40\%} & \accgain{$79.39 \pm 0.31$}{+3.45\%} & \accgain{$83.19 \pm 0.42$}{+4.20\%} & \accgain{$85.92 \pm 0.51$}{+3.98\%} & \accgain{$80.34 \pm 0.44$}{+6.54\%} & \accgain{$85.32 \pm 0.43$}{+5.69\%} \\
\bottomrule
\end{tabular}
}
\end{table*}

To address \textbf{Q1}, Tables~\ref{tab:comparison-rho05-a} and~\ref{tab:comparison-rho05-b} compare \ourmethod{} with representative FGL, missing-modality FL, and prototype-based FL baselines at $\rho=0.5$. These tasks jointly test whether graph-aware latent semantic synthesis improves both graph-structured prediction and readouts that directly depend on cross-modal semantics.
In Table~\ref{tab:comparison-rho05-a}, \ourmethod{} achieves the best reported results for node classification and modality retrieval across the completed client settings. In node classification, the method improves Movies from the strongest baseline score of $55.99\%$ to $65.74\%$ under Client $=5$, and it also remains strongest on Grocery for both client counts. This pattern indicates that the synthesized latents are not only useful as modality placeholders, but also preserve label-discriminative information after graph propagation.
The modality retrieval columns provide a complementary view of the same representation quality. Retrieval directly tests whether the visual and textual latent spaces remain comparable after one modality has been recovered. \ourmethod{} obtains the best R@5 on Toys and Flickr30k, including $79.40$ on Toys under Client $=10$. The gains over missing-modality FL baselines suggest that graph-conditioned synthesis and federated modality prototypes are both needed. Notably, retrieval benefits from cross-client semantic anchors that keep the two modality spaces aligned.

In Table~\ref{tab:comparison-rho05-b}, \ourmethod{} achieves the strongest results on link prediction and modality matching. Link prediction evaluates whether recovered node semantics remain informative for pairwise structural inference, and the gains on DY and Bili Dance indicate that the representations learned by \ourmethod{} support edge-level prediction beyond node-level semantic recovery. The improvements are moderate yet consistent, which is reasonable because topology-aware FGL baselines already exploit structural information for this task.
Modality matching places the strongest pressure on cross-modal semantic consistency. On KU and Bili Food, \ourmethod{} achieves the best AUC in all reported client settings, including $85.92$ on KU under Client $=10$ and $85.32$ on Bili Food under Client $=10$. These results support the central design choice that missing-modality recovery should be conditioned jointly on local graph context and modality-specific federated prototypes, rather than relying only on local observed features or generic prototype regularization.
\begin{table}[!htbp]
\setlength{\abovecaptionskip}{2pt}
\setlength{\belowcaptionskip}{2pt}
\caption{Ablation at Client $=5$ and $\rho=0.5$. Movies/Acc. and Toys/R@5 match Table~\ref{tab:dataset-summary}, and variant names follow Secs.~\ref{sec:availability-aware-graph-encoding}--\ref{sec:reliability-calibrated-semantic-fusion}.}
\label{tab:ablation-study}
\centering
\tiny
\setlength{\tabcolsep}{1.6pt}
\renewcommand{\arraystretch}{1.02}
\resizebox{\columnwidth}{!}{%
\begin{tabular}{@{}ccccc@{}}
\toprule
\textbf{Variant} & \textbf{Movies Acc.} & \textbf{Discrepancy} & \textbf{Toys R@5} & \textbf{Discrepancy} \\
\midrule
w/o Avail\_Graph Encoding & $61.44 \pm 0.36$ & $4.30$ & $68.89 \pm 0.42$ & $4.41$ \\
w/o Proto\_Latent Synthesis & $62.43 \pm 0.43$ & $3.31$ & $70.30 \pm 0.46$ & $3.00$ \\
w/o Reliab\_Semantic Fusion & $63.36 \pm 0.37$ & $2.38$ & $71.82 \pm 0.36$ & $1.48$ \\
\midrule
\rowcolor{FedMGSRow}
\textbf{\ourmethod{} (Full)} & $\mathbf{65.74 \pm 0.31}$ & -- & $\mathbf{73.30 \pm 0.53}$ & -- \\
\bottomrule
\end{tabular}
}
\end{table}
\begin{table}[!htbp]
\caption{Comparison on per-round complexity between the shared baselines in Tables~\ref{tab:comparison-rho05-a}--\ref{tab:nodecls-k10-ref} and FedMGS.}
\label{tab:complexity-analysis}
\centering
\tiny
\setlength{\tabcolsep}{2pt}
\renewcommand{\arraystretch}{1.28}
\resizebox{0.95\columnwidth}{!}{%
\begin{tabular}{@{}cccc@{}}
\toprule
\makecell[c]{Method} & \makecell[c]{Extra client computation} & \makecell[c]{Extra server computation} & \makecell[c]{Extra communication} \\
\midrule
FedAvg
& \makecell[c]{None beyond local training}
& $\mathcal{O}(SP_{\theta})$
& $\mathcal{O}(P_{\theta})$ \\
FedProto
& $\mathcal{O}(nCd)$
& $\mathcal{O}(SP_{\theta}+SCd)$
& $\mathcal{O}(P_{\theta}+Cd)$ \\
FedMAC
& $\mathcal{O}(Mnd^2+M^2nd)$
& $\mathcal{O}(SP_{\theta})$
& $\mathcal{O}(P_{\theta})$ \\
FedMVP
& $\mathcal{O}(Mnd^2)$
& $\mathcal{O}(SP_{\theta})$
& $\mathcal{O}(P_{\theta})$ \\
FedLap
& \makecell[c]{$\mathcal{O}(L_g e d+nd^2)$\\plus spectral preprocessing}
& $\mathcal{O}(SP_{\theta})$
& $\mathcal{O}(P_{\theta})$ \\
FedGTA
& $\mathcal{O}(L_g e C+QnC)$
& $\mathcal{O}(S^2QC+S^2P_{\theta})$
& $\mathcal{O}(P_{\theta}+QC)$ \\
mh-pflid
& $\mathcal{O}(2L_g e d+2nd^2+rP_{\theta})$
& $\mathcal{O}(SrP_{\theta})$
& $\mathcal{O}(rP_{\theta})$ \\
Pepsy
& $\mathcal{O}(nd^2)$
& $\mathcal{O}(SP_{\theta})$
& $\mathcal{O}(P_{\theta})$ \\
\rowcolor{FedMGSRow}
\textbf{\ourmethod{} (Ours)}
& $\mathcal{O}(L_g e d+Mnd^2+nCd)$
& $\mathcal{O}(SP_{\theta}+SMCd+MCd)$
& $\mathcal{O}(P_{\theta}+MCd+MC)$ \\
\bottomrule
\end{tabular}
}
\vspace{-1cm}
\end{table}
\subsection{Robustness Analysis (Answer for Q2)}
\label{sec:robustness-analysis}

To address \textbf{Q2}, Table~\ref{tab:nodecls-k10-ref} reports representative results across missingness rates and task families. This stress test asks whether \ourmethod{} remains useful as incomplete modalities become milder or more severe under the task-specific Client-level or Node-level missingness protocol.

\begin{table*}[t]
\caption{Representative robustness results across missingness rates. Results are reported as mean $\pm$ std. Best results are in \textbf{bold}, and second-best results are underlined.}
\vspace{-0.3cm}
\label{tab:nodecls-k10-ref}
\centering
\scriptsize
\setlength{\tabcolsep}{1.15pt}
\renewcommand{\arraystretch}{1.15}
\resizebox{\textwidth}{!}{%
\begin{tabular}{@{}lcccccccccccc@{}}
\toprule
\multirow{2}{*}{Method} & \multicolumn{3}{c}{Movies (Acc)} & \multicolumn{3}{c}{DY (AUC)} & \multicolumn{3}{c}{KU (AUC)} & \multicolumn{3}{c}{Toys (R@5)} \\
\cmidrule(lr){2-4}\cmidrule(lr){5-7}\cmidrule(lr){8-10}\cmidrule(l){11-13}
 & $\rho=0.3$ & $\rho=0.5$ & $\rho=0.7$ & $\rho=0.3$ & $\rho=0.5$ & $\rho=0.7$ & $\rho=0.3$ & $\rho=0.5$ & $\rho=0.7$ & $\rho=0.3$ & $\rho=0.5$ & $\rho=0.7$ \\
\midrule
FedAvg & $55.20 \pm 0.58$ & $54.03 \pm 0.53$ & $52.40 \pm 0.76$ & $76.30 \pm 0.39$ & $74.79 \pm 0.43$ & $73.20 \pm 0.52$ & $75.00 \pm 0.57$ & $73.31 \pm 0.68$ & $71.60 \pm 0.84$ & $65.50 \pm 0.58$ & $64.22 \pm 0.70$ & $62.80 \pm 0.81$ \\
FedProto & $54.80 \pm 0.58$ & $53.45 \pm 0.55$ & $51.80 \pm 0.73$ & $77.10 \pm 0.41$ & $75.88 \pm 0.41$ & $74.60 \pm 0.67$ & $74.30 \pm 0.55$ & $72.74 \pm 0.61$ & $70.80 \pm 0.74$ & $63.80 \pm 0.72$ & $62.47 \pm 0.80$ & $61.10 \pm 0.95$ \\
FedMVP & \secondcell{$57.00 \pm 0.50$} & \secondcell{$55.99 \pm 0.55$} & $53.70 \pm 0.71$ & $78.20 \pm 0.34$ & $77.43 \pm 0.38$ & $76.00 \pm 0.52$ & $77.20 \pm 0.51$ & $75.74 \pm 0.64$ & $74.00 \pm 0.71$ & $66.90 \pm 0.57$ & $65.42 \pm 0.70$ & $63.80 \pm 0.85$ \\
FedMAC & $54.40 \pm 0.56$ & $53.12 \pm 0.55$ & $51.60 \pm 0.72$ & $77.80 \pm 0.44$ & $76.75 \pm 0.44$ & $75.30 \pm 0.57$ & $81.00 \pm 0.53$ & $79.53 \pm 0.65$ & $77.70 \pm 0.75$ & $64.10 \pm 0.61$ & $62.83 \pm 0.71$ & $61.50 \pm 0.92$ \\
FedLap & $53.10 \pm 0.56$ & $52.40 \pm 0.66$ & $50.90 \pm 0.74$ & $76.20 \pm 0.53$ & $74.98 \pm 0.56$ & $73.50 \pm 0.62$ & $77.00 \pm 0.67$ & $75.22 \pm 0.66$ & $73.80 \pm 0.83$ & $62.60 \pm 0.73$ & $61.15 \pm 0.79$ & $59.80 \pm 0.95$ \\
FedGTA & $56.20 \pm 0.45$ & $55.25 \pm 0.54$ & \secondcell{$53.90 \pm 0.75$} & \secondcell{$79.00 \pm 0.34$} & \secondcell{$78.42 \pm 0.42$} & \secondcell{$76.90 \pm 0.56$} & $80.00 \pm 0.53$ & $78.54 \pm 0.67$ & $76.90 \pm 0.81$ & $68.30 \pm 0.59$ & $66.95 \pm 0.74$ & $65.20 \pm 0.85$ \\
MH-pFLID & $54.90 \pm 0.59$ & $53.76 \pm 0.52$ & $52.30 \pm 0.65$ & $77.00 \pm 0.40$ & $76.03 \pm 0.43$ & $74.70 \pm 0.58$ & $78.40 \pm 0.63$ & $76.87 \pm 0.69$ & $75.30 \pm 0.83$ & $69.10 \pm 0.59$ & $67.76 \pm 0.79$ & $66.10 \pm 0.90$ \\
PEPSY & $53.40 \pm 0.47$ & $51.96 \pm 0.58$ & $50.70 \pm 0.74$ & $77.30 \pm 0.46$ & $76.45 \pm 0.53$ & $75.00 \pm 0.57$ & \secondcell{$81.20 \pm 0.59$} & \secondcell{$79.84 \pm 0.70$} & \secondcell{$78.20 \pm 0.72$} & \secondcell{$70.00 \pm 0.67$} & \secondcell{$68.60 \pm 0.78$} & \secondcell{$66.90 \pm 0.82$} \\
\midrule
\rowcolor{FedMGSRow}
\textbf{\ourmethod{} (Ours)} & \accgain{$68.95 \pm 0.35$}{+20.96\%} & \accgain{$65.74 \pm 0.44$}{+17.41\%} & \accgain{$63.53 \pm 0.60$}{+17.87\%} & \accgain{$80.12 \pm 0.21$}{+1.42\%} & \accgain{$79.90 \pm 0.30$}{+1.89\%} & \accgain{$78.32 \pm 0.45$}{+1.85\%} & \accgain{$84.21 \pm 0.56$}{+3.71\%} & \accgain{$83.19 \pm 0.45$}{+4.20\%} & \accgain{$81.45 \pm 0.67$}{+4.16\%} & \accgain{$74.24 \pm 0.51$}{+6.06\%} & \accgain{$73.30 \pm 0.54$}{+6.85\%} & \accgain{$71.23 \pm 0.72$}{+6.47\%} \\
\bottomrule
\end{tabular}
}
\end{table*}
\begin{figure*}[t]
  \centering
  \includegraphics[width=0.8\textwidth]{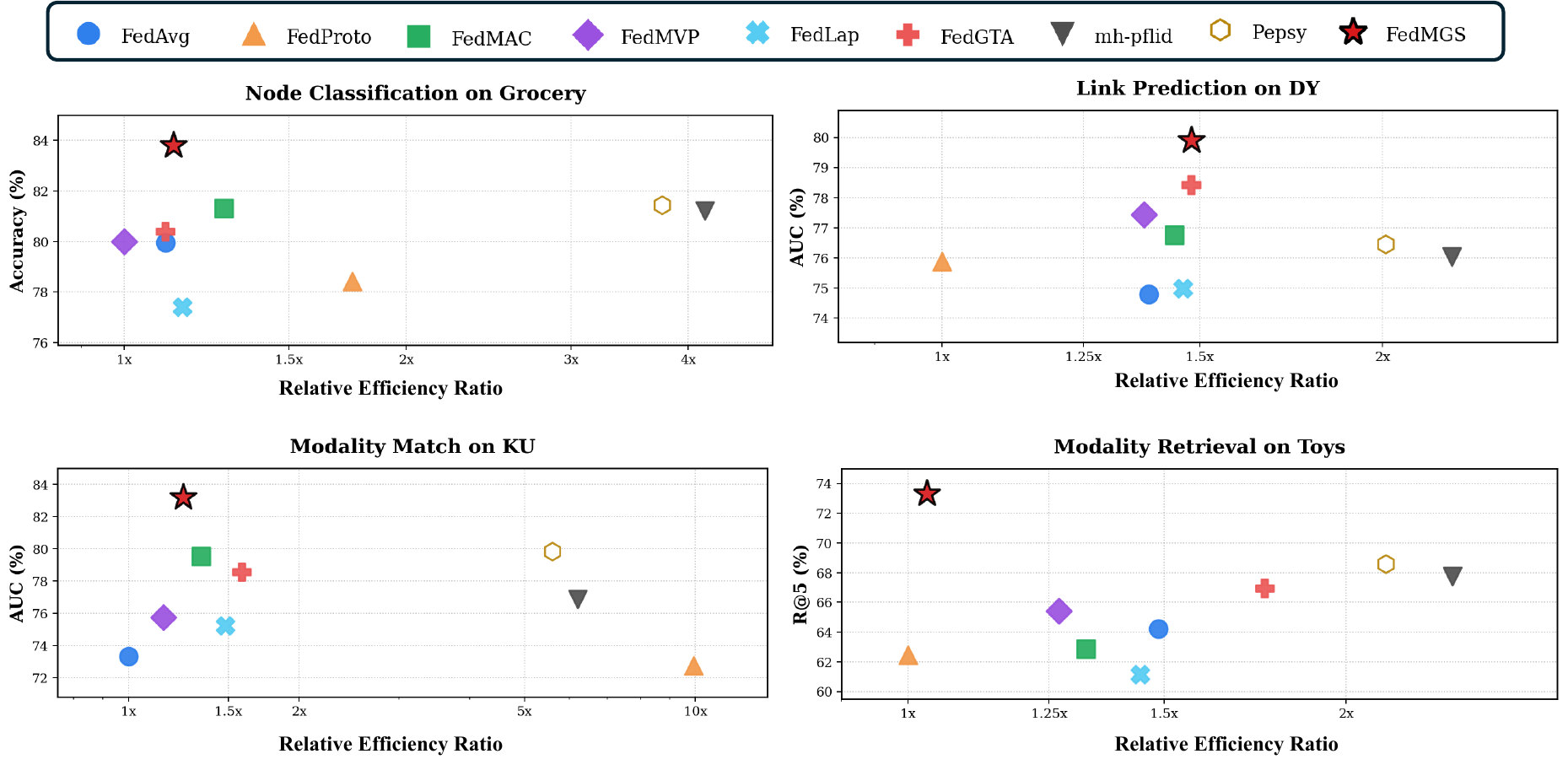}
  \vspace{-0.5cm}
  \caption{Efficiency-performance tradeoff across the four tasks. The relative efficiency ratio is normalized by the fastest-running baseline for each task, so lower ratios indicate less running-time overhead.}
  \label{fig:efficiency-test}
\end{figure*}

Across Table~\ref{tab:nodecls-k10-ref}, \ourmethod{} retains the best score in all reported representative columns from $\rho=0.3$ to $\rho=0.7$. The absolute scores generally decline as $\rho$ increases, but the relative advantage remains consistent across node classification, link prediction, modality matching, and modality retrieval. This pattern supports the intended robustness mechanism. Availability-aware encoding reduces contamination from absent modalities, while graph-conditioned synthesis and modality-specific prototypes provide usable semantic evidence when either local nodes or entire clients have incomplete modality observations.

The robustness trend is especially informative because the table mixes Client-level and Node-level missingness. Movies and Toys evaluate settings where selected clients lose an entire modality, so local training can become modality-biased if the model treats client observations as complete. DY and KU evaluate node-level incompleteness, where the model must avoid propagating invalid node attributes through local neighborhoods. Maintaining the best reported scores at $\rho=0.7$ on all four representative datasets suggests that \ourmethod{} handles both forms of missingness with the same recovery principle. Only observed modalities seed graph encoding, and missing latents are reconstructed from graph context plus class-modality prototypes.

\subsection{Ablation Study (Answer for Q3)}
\label{sec:ablation-study}

To address \textbf{Q3}, Table~\ref{tab:ablation-study} removes the three core components of \ourmethod{} under the same dataset-task-metric alignment as Table~\ref{tab:dataset-summary}. The short variant names denote availability-aware graph encoding (Avail\_Graph Encoding, Sec.~\ref{sec:availability-aware-graph-encoding}), prototype-guided latent semantic synthesis (Proto\_Latent Synthesis, Sec.~\ref{sec:prototype-guided-latent-semantic-synthesis}), and reliability-calibrated semantic fusion (Reliab\_Semantic Fusion, Sec.~\ref{sec:reliability-calibrated-semantic-fusion}). The discrepancy values are absolute drops from the full \ourmethod{} model.

Removing any component reduces performance on both reported client-level missingness settings, confirming that the three modules contribute complementary evidence. The largest drop occurs when Avail\_Graph Encoding is removed, which is consistent with the method design. If absent modality features enter propagation, the local graph encoder can spread unreliable semantics before synthesis begins. Removing Proto\_Latent Synthesis also causes clear degradation because clients with missing modalities lose access to cross-client class-modality anchors. Reliab\_Semantic Fusion has a smaller but still measurable effect, indicating that the recovered latent should not be treated as uniformly reliable even when the synthesis module is present.
\begin{figure}[t]
  \centering
  \includegraphics[width=\columnwidth]{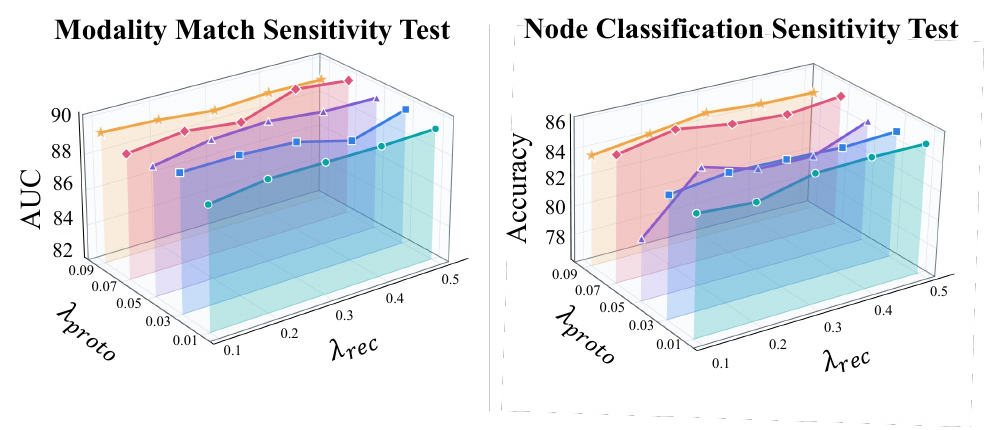}
  \caption{Hyperparameter sensitivity over Reconstruction loss weight, $\lambda_{\mathrm{rec}}$, and Prototype loss weight $\lambda_{\mathrm{rec}}$ for modality matching on KU and node classification on Grocery.}
  \label{fig:hyperparameter-sensitivity}
\end{figure}

\subsection{Hyperparameter Analysis (Answer for Q4)}
\label{sec:hyperparameter-analysis}

To address \textbf{Q4}, Figure~\ref{fig:hyperparameter-sensitivity} varies the reconstruction weight $\lambda_{\mathrm{rec}}$ and the prototype loss weight $\lambda_{\mathrm{proto}}$ for modality matching and node classification. The reconstruction weight controls how strongly synthesized latent representations are constrained by observed modality representations, while the prototype loss weight controls the strength of semantic alignment between local representations and federated prototype references.

\ourmethod{} changes smoothly under different hyperparameter settings, suggesting that its performance is not tied to a narrow tuning range. Stronger regions generally appear when prototype alignment receives a relatively large weight. Modality matching favors strong prototype alignment with a conservative-to-moderate reconstruction weight, whereas node classification benefits from strong prototype alignment together with moderate-to-larger reconstruction regularization. This trend is consistent with the design of \ourmethod{}: prototype-conditioned synthesis and prototype regularization stabilize cross-client semantics, while reconstruction keeps synthesized latents close to observed modality structure without overwhelming local graph context or task supervision.

The figure also provides a practical tuning guideline. Prototype-related regularization should remain active because it supplies cross-client semantic references used by synthesis and alignment, while $\lambda_{\mathrm{rec}}$ should be treated as a regularization weight rather than the main training signal. The smooth response surface in Figure~\ref{fig:hyperparameter-sensitivity} indicates that moderate-to-strong settings are sufficient without precise dataset-specific tuning.

\subsection{Efficiency Analysis (Answer for Q5)}
\label{sec:efficiency-analysis}

To address \textbf{Q5}, Figure~\ref{fig:efficiency-test} relates task performance to relative running-time overhead. For each task, the efficiency ratio is normalized by the fastest baseline, so smaller values indicate lower overhead within that task. Table~\ref{tab:complexity-analysis} complements this comparison with per-round complexity between FedMGS and other baselines, where $|V|$, $|E|$, $C$, and $d$ denote the numbers of local nodes, local edges, classes, and prototype dimensions, respectively.

The results indicate that \ourmethod{} obtains its accuracy gains with limited communication overhead relative to the baselines. Its additional cost mainly comes from local graph propagation and latent synthesis on clients, together with server-side aggregation of class-modality prototypes. The transmitted payload remains compact: clients send model parameters, prototype vectors ($2Cd$ scalars), observation counts ($2C$ scalars), and one sample count. Since these prototype statistics scale with classes and latent dimensions rather than local graph size, communication remains dominated by model parameters instead of raw data or synthesized explicit infomration, which further strengthen the privacy preservation needs. 

\section{Conclusion}

This paper studies modality-imbalanced MM-FGL, where clients train over multimodal graphs while visual or textual evidence may be absent at the client level or node level. FedMGS addresses this setting by framing missing modalities as graph-aware latent semantic synthesis rather than raw feature reconstruction. Its availability-aware graph encoder prevents unavailable modalities from entering message passing, its prototype-guided latent synthesizer imports class-modality semantics from federated prototype banks, and its reliability-calibrated fusion module regulates the influence of recovered latents before downstream readout. This design keeps raw features, topology, and node embeddings local while allowing clients to share compact semantic anchors.
The experimental results support the proposed design across graph-supervised and cross-modal tasks.FedMGS achieves the best reported results over four downstream tasks and the Robustness test indicates that the method remains effective as missingness becomes milder or more severe.
These findings indicate that graph-conditioned latent recovery is a practical solution. Future MM-FGL research can move beyond controlled missingness by studying dynamically evolving, and policy-constrained modalities in real deployments.

\section{GenAI Usage Disclosure}

Generative AI tools were used solely for language polishing, grammar checking, and improving the readability of the manuscript. They were not used for research ideation, method design, code generation, data processing, experiment execution, result analysis, or scientific claim generation. All research content, implementation, experiments, results, and conclusions were independently developed and verified by the authors. The authors take full responsibility for the content of this paper.
\bibliographystyle{ACM-Reference-Format}
\bibliography{references}


\begin{thebibliography}{48}


\ifx \showCODEN    \undefined \def \showCODEN     #1{\unskip}     \fi
\ifx \showISBNx    \undefined \def \showISBNx     #1{\unskip}     \fi
\ifx \showISBNxiii \undefined \def \showISBNxiii  #1{\unskip}     \fi
\ifx \showISSN     \undefined \def \showISSN      #1{\unskip}     \fi
\ifx \showLCCN     \undefined \def \showLCCN      #1{\unskip}     \fi
\ifx \shownote     \undefined \def \shownote      #1{#1}          \fi
\ifx \showarticletitle \undefined \def \showarticletitle #1{#1}   \fi
\ifx \showURL      \undefined \def \showURL       {\relax}        \fi
\providecommand\bibfield[2]{#2}
\providecommand\bibinfo[2]{#2}
\providecommand\natexlab[1]{#1}
\providecommand\showeprint[2][]{arXiv:#2}

\bibitem[Aliakbari et~al\mbox{.}(2025)]%
        {aliakbari2025fedlap}
\bibfield{author}{\bibinfo{person}{Javad Aliakbari}, \bibinfo{person}{Johan {\"O}stman}, \bibinfo{person}{Ashkan Panahi}, {and} \bibinfo{person}{Alexandre {Graell i Amat}}.} \bibinfo{year}{2025}\natexlab{}.
\newblock \showarticletitle{Subgraph Federated Learning via Spectral Methods}. In \bibinfo{booktitle}{\emph{Advances in Neural Information Processing Systems}}, Vol.~\bibinfo{volume}{38}.
\newblock
\urldef\tempurl%
\url{https://arxiv.org/abs/2510.25657}
\showURL{%
\tempurl}


\bibitem[Baek et~al\mbox{.}(2023)]%
        {baek2022fedpub}
\bibfield{author}{\bibinfo{person}{Jinheon Baek}, \bibinfo{person}{Wonyong Jeong}, \bibinfo{person}{Jiongdao Jin}, \bibinfo{person}{Jaehong Yoon}, {and} \bibinfo{person}{Sung~Ju Hwang}.} \bibinfo{year}{2023}\natexlab{}.
\newblock \showarticletitle{Personalized Subgraph Federated Learning}. In \bibinfo{booktitle}{\emph{Proceedings of the 40th International Conference on Machine Learning}}.
\newblock


\bibitem[Bang et~al\mbox{.}(2023)]%
        {bang2023app_gnn_bio1}
\bibfield{author}{\bibinfo{person}{Dongmin Bang}, \bibinfo{person}{Sangsoo Lim}, \bibinfo{person}{Sangseon Lee}, {and} \bibinfo{person}{Sun Kim}.} \bibinfo{year}{2023}\natexlab{}.
\newblock \showarticletitle{Biomedical knowledge graph learning for drug repurposing by extending guilt-by-association to multiple layers}.
\newblock \bibinfo{journal}{\emph{Nature Communications}} \bibinfo{volume}{14}, \bibinfo{number}{1} (\bibinfo{year}{2023}), \bibinfo{pages}{3570}.
\newblock


\bibitem[Blondel et~al\mbox{.}(2008)]%
        {blondel2008louvain}
\bibfield{author}{\bibinfo{person}{Vincent~D. Blondel}, \bibinfo{person}{Jean-Loup Guillaume}, \bibinfo{person}{Renaud Lambiotte}, {and} \bibinfo{person}{Etienne Lefebvre}.} \bibinfo{year}{2008}\natexlab{}.
\newblock \showarticletitle{Fast unfolding of communities in large networks}.
\newblock \bibinfo{journal}{\emph{Journal of Statistical Mechanics: Theory and Experiment}} \bibinfo{volume}{2008}, \bibinfo{number}{10} (\bibinfo{year}{2008}), \bibinfo{pages}{P10008}.
\newblock


\bibitem[Cai et~al\mbox{.}(2021)]%
        {cai2021link_prediction2}
\bibfield{author}{\bibinfo{person}{Lei Cai}, \bibinfo{person}{Jundong Li}, \bibinfo{person}{Jie Wang}, {and} \bibinfo{person}{Shuiwang Ji}.} \bibinfo{year}{2021}\natexlab{}.
\newblock \showarticletitle{Line Graph Neural Networks for Link Prediction}.
\newblock \bibinfo{journal}{\emph{IEEE Transactions on Pattern Analysis and Machine Intelligence}} (\bibinfo{year}{2021}).
\newblock


\bibitem[Cai et~al\mbox{.}(2023)]%
        {cai2023app_gnn_rec3}
\bibfield{author}{\bibinfo{person}{Xuheng Cai}, \bibinfo{person}{Chao Huang}, \bibinfo{person}{Lianghao Xia}, {and} \bibinfo{person}{Xubin Ren}.} \bibinfo{year}{2023}\natexlab{}.
\newblock \showarticletitle{LightGCL: Simple Yet Effective Graph Contrastive Learning for Recommendation}. In \bibinfo{booktitle}{\emph{International Conference on Learning Representations}}.
\newblock


\bibitem[Che et~al\mbox{.}(2024)]%
        {che2024fedmvp}
\bibfield{author}{\bibinfo{person}{Liwei Che}, \bibinfo{person}{Jiaqi Wang}, \bibinfo{person}{Xinyue Liu}, {and} \bibinfo{person}{Fenglong Ma}.} \bibinfo{year}{2024}\natexlab{}.
\newblock \bibinfo{title}{Leveraging Foundation Models for Multi-modal Federated Learning with Incomplete Modality}.
\newblock \bibinfo{howpublished}{arXiv preprint arXiv:2406.11048}.
\newblock
\href{https://doi.org/10.48550/arXiv.2406.11048}{doi:\nolinkurl{10.48550/arXiv.2406.11048}}


\bibitem[Chen et~al\mbox{.}(2021)]%
        {chen2021fedgl}
\bibfield{author}{\bibinfo{person}{Chuan Chen}, \bibinfo{person}{Weibo Hu}, \bibinfo{person}{Ziyue Xu}, {and} \bibinfo{person}{Zibin Zheng}.} \bibinfo{year}{2021}\natexlab{}.
\newblock \showarticletitle{FedGL: Federated Graph Learning Framework with Global Self-Supervision}.
\newblock \bibinfo{journal}{\emph{arXiv preprint arXiv:2105.03170}} (\bibinfo{year}{2021}).
\newblock


\bibitem[Chen et~al\mbox{.}(2026)]%
        {zhang2026stage}
\bibfield{author}{\bibinfo{person}{Zekai Chen}, \bibinfo{person}{Xun Wu}, \bibinfo{person}{Xunkai Li}, \bibinfo{person}{Yihan Sun}, \bibinfo{person}{Rong-Hua Li}, {and} \bibinfo{person}{Guoren Wang}.} \bibinfo{year}{2026}\natexlab{}.
\newblock \bibinfo{title}{STAGE: Tackling Semantic Drift in Multimodal Federated Graph Learning}.
\newblock \bibinfo{howpublished}{arXiv preprint arXiv:2605.11919}.
\newblock
\href{https://doi.org/10.48550/arXiv.2605.11919}{doi:\nolinkurl{10.48550/arXiv.2605.11919}}


\bibitem[Feng et~al\mbox{.}(2023)]%
        {feng2023fedmultimodal}
\bibfield{author}{\bibinfo{person}{Tiantian Feng}, \bibinfo{person}{Digbalay Bose}, \bibinfo{person}{Tuo Zhang}, \bibinfo{person}{Rajat Hebbar}, \bibinfo{person}{Anil Ramakrishna}, \bibinfo{person}{Rahul Gupta}, \bibinfo{person}{Mi Zhang}, \bibinfo{person}{Salman Avestimehr}, {and} \bibinfo{person}{Shrikanth Narayanan}.} \bibinfo{year}{2023}\natexlab{}.
\newblock \showarticletitle{FedMultimodal: A Benchmark for Multimodal Federated Learning}. In \bibinfo{booktitle}{\emph{Proceedings of the 29th ACM SIGKDD Conference on Knowledge Discovery and Data Mining}}. \bibinfo{publisher}{Association for Computing Machinery}, \bibinfo{address}{New York, NY, USA}, \bibinfo{pages}{4035--4045}.
\newblock
\href{https://doi.org/10.1145/3580305.3599825}{doi:\nolinkurl{10.1145/3580305.3599825}}


\bibitem[Fu et~al\mbox{.}(2022)]%
        {fu2022fgl_survey_1}
\bibfield{author}{\bibinfo{person}{Xingbo Fu}, \bibinfo{person}{Binchi Zhang}, \bibinfo{person}{Yushun Dong}, \bibinfo{person}{Chen Chen}, {and} \bibinfo{person}{Jundong Li}.} \bibinfo{year}{2022}\natexlab{}.
\newblock \showarticletitle{Federated graph machine learning: A survey of concepts, techniques, and applications}.
\newblock \bibinfo{journal}{\emph{ACM SIGKDD Explorations Newsletter}} \bibinfo{volume}{24}, \bibinfo{number}{2} (\bibinfo{year}{2022}), \bibinfo{pages}{32--47}.
\newblock


\bibitem[Hamilton et~al\mbox{.}(2017)]%
        {hamilton2017graphsage}
\bibfield{author}{\bibinfo{person}{Will Hamilton}, \bibinfo{person}{Zhitao Ying}, {and} \bibinfo{person}{Jure Leskovec}.} \bibinfo{year}{2017}\natexlab{}.
\newblock \showarticletitle{Inductive Representation Learning on Large Graphs}.
\newblock \bibinfo{journal}{\emph{Advances in Neural Information Processing Systems}} (\bibinfo{year}{2017}).
\newblock


\bibitem[He et~al\mbox{.}(2021)]%
        {he2021fedgraphnn}
\bibfield{author}{\bibinfo{person}{Chaoyang He}, \bibinfo{person}{Keshav Balasubramanian}, \bibinfo{person}{Emir Ceyani}, \bibinfo{person}{Carl Yang}, \bibinfo{person}{Han Xie}, \bibinfo{person}{Lichao Sun}, \bibinfo{person}{Lifang He}, \bibinfo{person}{Liangwei Yang}, \bibinfo{person}{Philip~S. Yu}, \bibinfo{person}{Yu Rong}, {et~al\mbox{.}}} \bibinfo{year}{2021}\natexlab{}.
\newblock \showarticletitle{FedGraphNN: A Federated Learning Benchmark System for Graph Neural Networks}. In \bibinfo{booktitle}{\emph{International Conference on Learning Representations Workshop on Distributed and Private Machine Learning}}.
\newblock


\bibitem[He et~al\mbox{.}(2025)]%
        {he2025unigraph2}
\bibfield{author}{\bibinfo{person}{Yufei He}, \bibinfo{person}{Yuan Sui}, \bibinfo{person}{Xiaoxin He}, \bibinfo{person}{Yue Liu}, \bibinfo{person}{Yifei Sun}, {and} \bibinfo{person}{Bryan Hooi}.} \bibinfo{year}{2025}\natexlab{}.
\newblock \bibinfo{title}{UniGraph2: Learning a Unified Embedding Space to Bind Multimodal Graphs}.
\newblock \bibinfo{howpublished}{arXiv preprint arXiv:2502.00806}.
\newblock
\href{https://doi.org/10.48550/arXiv.2502.00806}{doi:\nolinkurl{10.48550/arXiv.2502.00806}}


\bibitem[Hyun et~al\mbox{.}(2023)]%
        {hyun2023app_gnn_fina2}
\bibfield{author}{\bibinfo{person}{Woochang Hyun}, \bibinfo{person}{Jaehong Lee}, {and} \bibinfo{person}{Bongwon Suh}.} \bibinfo{year}{2023}\natexlab{}.
\newblock \showarticletitle{Anti-Money Laundering in Cryptocurrency via Multi-Relational Graph Neural Network}. In \bibinfo{booktitle}{\emph{Pacific-Asia Conference on Knowledge Discovery and Data Mining}}. Springer, \bibinfo{pages}{118--130}.
\newblock


\bibitem[Kipf and Welling(2017)]%
        {kipf2017gcn}
\bibfield{author}{\bibinfo{person}{Thomas~N. Kipf} {and} \bibinfo{person}{Max Welling}.} \bibinfo{year}{2017}\natexlab{}.
\newblock \bibinfo{title}{Semi-Supervised Classification with Graph Convolutional Networks}.
\newblock \bibinfo{howpublished}{International Conference on Learning Representations (ICLR)}.
\newblock
\urldef\tempurl%
\url{https://openreview.net/forum?id=SJU4ayYgl}
\showURL{%
\tempurl}


\bibitem[Le et~al\mbox{.}(2024)]%
        {le2025mfcpl}
\bibfield{author}{\bibinfo{person}{Huy~Q. Le}, \bibinfo{person}{Chu~Myaet Thwal}, \bibinfo{person}{Yu Qiao}, \bibinfo{person}{Ye~Lin Tun}, \bibinfo{person}{Minh N.~H. Nguyen}, {and} \bibinfo{person}{Choong~Seon Hong}.} \bibinfo{year}{2024}\natexlab{}.
\newblock \bibinfo{title}{Cross-Modal Prototype based Multimodal Federated Learning under Severely Missing Modality}.
\newblock \bibinfo{howpublished}{arXiv preprint arXiv:2401.13898}.
\newblock
\href{https://doi.org/10.48550/arXiv.2401.13898}{doi:\nolinkurl{10.48550/arXiv.2401.13898}}


\bibitem[Li et~al\mbox{.}(2026)]%
        {li2026mmopenfgl}
\bibfield{author}{\bibinfo{person}{Xunkai Li}, \bibinfo{person}{Yuming Ai}, \bibinfo{person}{Yinlin Zhu}, \bibinfo{person}{Haodong Lu}, \bibinfo{person}{Yi Zhang}, \bibinfo{person}{Guohao Fu}, \bibinfo{person}{Bowen Fan}, \bibinfo{person}{Qiangqiang Dai}, \bibinfo{person}{Rong-Hua Li}, {and} \bibinfo{person}{Guoren Wang}.} \bibinfo{year}{2026}\natexlab{}.
\newblock \bibinfo{title}{MM-OpenFGL: A Comprehensive Benchmark for Multimodal Federated Graph Learning}.
\newblock \bibinfo{howpublished}{arXiv preprint arXiv:2601.22416}.
\newblock
\href{https://doi.org/10.48550/arXiv.2601.22416}{doi:\nolinkurl{10.48550/arXiv.2601.22416}}


\bibitem[Li et~al\mbox{.}(2024a)]%
        {li2024adafgl}
\bibfield{author}{\bibinfo{person}{Xunkai Li}, \bibinfo{person}{Zhengyu Wu}, \bibinfo{person}{Wentao Zhang}, \bibinfo{person}{Henan Sun}, \bibinfo{person}{Rong-Hua Li}, {and} \bibinfo{person}{Guoren Wang}.} \bibinfo{year}{2024}\natexlab{a}.
\newblock \bibinfo{title}{AdaFGL: A New Paradigm for Federated Node Classification with Topology Heterogeneity}.
\newblock \bibinfo{howpublished}{arXiv preprint arXiv:2401.11750}.
\newblock
\href{https://doi.org/10.48550/arXiv.2401.11750}{doi:\nolinkurl{10.48550/arXiv.2401.11750}}


\bibitem[Li et~al\mbox{.}(2024b)]%
        {li2024fedgta}
\bibfield{author}{\bibinfo{person}{Xunkai Li}, \bibinfo{person}{Zhengyu Wu}, \bibinfo{person}{Wentao Zhang}, \bibinfo{person}{Yinlin Zhu}, \bibinfo{person}{Rong-Hua Li}, {and} \bibinfo{person}{Guoren Wang}.} \bibinfo{year}{2024}\natexlab{b}.
\newblock \showarticletitle{FedGTA: Topology-Aware Averaging for Federated Graph Learning}.
\newblock \bibinfo{journal}{\emph{Proceedings of the VLDB Endowment}} \bibinfo{volume}{17}, \bibinfo{number}{1} (\bibinfo{year}{2024}), \bibinfo{pages}{41--50}.
\newblock


\bibitem[McMahan et~al\mbox{.}(2017)]%
        {mcmahan2017fedavg}
\bibfield{author}{\bibinfo{person}{H.~Brendan McMahan}, \bibinfo{person}{Eider Moore}, \bibinfo{person}{Daniel Ramage}, \bibinfo{person}{Seth Hampson}, {and} \bibinfo{person}{Blaise~Ag{\"u}era y Arcas}.} \bibinfo{year}{2017}\natexlab{}.
\newblock \showarticletitle{Communication-Efficient Learning of Deep Networks from Decentralized Data}. In \bibinfo{booktitle}{\emph{Proceedings of the 20th International Conference on Artificial Intelligence and Statistics}} \emph{(\bibinfo{series}{Proceedings of Machine Learning Research}, Vol.~\bibinfo{volume}{54})}. \bibinfo{publisher}{PMLR}, \bibinfo{address}{Fort Lauderdale, FL, USA}, \bibinfo{pages}{1273--1282}.
\newblock
\urldef\tempurl%
\url{https://proceedings.mlr.press/v54/mcmahan17a.html}
\showURL{%
\tempurl}


\bibitem[Nguyen et~al\mbox{.}(2024)]%
        {nguyen2024fedmac}
\bibfield{author}{\bibinfo{person}{Manh~Duong Nguyen}, \bibinfo{person}{Trung~Thanh Nguyen}, \bibinfo{person}{Huy~Hieu Pham}, \bibinfo{person}{Trong~Nghia Hoang}, \bibinfo{person}{Phi~Le Nguyen}, {and} \bibinfo{person}{Thanh~Trung Huynh}.} \bibinfo{year}{2024}\natexlab{}.
\newblock \bibinfo{title}{FedMAC: Tackling Partial-Modality Missing in Federated Learning with Cross-Modal Aggregation and Contrastive Regularization}.
\newblock \bibinfo{howpublished}{arXiv preprint arXiv:2410.03070}.
\newblock
\href{https://doi.org/10.48550/arXiv.2410.03070}{doi:\nolinkurl{10.48550/arXiv.2410.03070}}


\bibitem[Nguyen et~al\mbox{.}(2025)]%
        {nguyen2025pepsy}
\bibfield{author}{\bibinfo{person}{Tan Nguyen} {et~al\mbox{.}}} \bibinfo{year}{2025}\natexlab{}.
\newblock \bibinfo{title}{PEPSY: Privacy-Preserving Embedding Controls for Heterogeneous Missing Modalities}.
\newblock \bibinfo{howpublished}{Preprint}.
\newblock


\bibitem[Ni et~al\mbox{.}(2019)]%
        {Movies_Grocery_toys}
\bibfield{author}{\bibinfo{person}{Jianmo Ni}, \bibinfo{person}{Jiacheng Li}, {and} \bibinfo{person}{Julian~J. McAuley}.} \bibinfo{year}{2019}\natexlab{}.
\newblock \showarticletitle{Justifying Recommendations using Distantly-Labeled Reviews and Fine-Grained Aspects}. In \bibinfo{booktitle}{\emph{{EMNLP/IJCNLP} {(1)}}}. \bibinfo{publisher}{Association for Computational Linguistics}, \bibinfo{pages}{188--197}.
\newblock


\bibitem[Pan et~al\mbox{.}(2022)]%
        {pan2022fedapp_gnn_fina1}
\bibfield{author}{\bibinfo{person}{Zhirui Pan}, \bibinfo{person}{Guangzhong Wang}, \bibinfo{person}{Zhaoning Li}, \bibinfo{person}{Lifeng Chen}, \bibinfo{person}{Yang Bian}, {and} \bibinfo{person}{Zhongyuan Lai}.} \bibinfo{year}{2022}\natexlab{}.
\newblock \showarticletitle{2SFGL: A Simple And Robust Protocol For Graph-Based Fraud Detection}. In \bibinfo{booktitle}{\emph{2022 IEEE International Conference on Cloud Computing Technology and Science}}. IEEE, \bibinfo{pages}{194--201}.
\newblock


\bibitem[Peng et~al\mbox{.}(2024)]%
        {peng2024fedmm}
\bibfield{author}{\bibinfo{person}{Yuanzhe Peng}, \bibinfo{person}{Jieming Bian}, {and} \bibinfo{person}{Jie Xu}.} \bibinfo{year}{2024}\natexlab{}.
\newblock \bibinfo{title}{FedMM: Federated Multi-Modal Learning with Modality Heterogeneity in Computational Pathology}.
\newblock \bibinfo{howpublished}{arXiv preprint arXiv:2402.15858}.
\newblock
\href{https://doi.org/10.48550/arXiv.2402.15858}{doi:\nolinkurl{10.48550/arXiv.2402.15858}}


\bibitem[Plummer et~al\mbox{.}(2015)]%
        {flickr30k}
\bibfield{author}{\bibinfo{person}{Bryan~A. Plummer}, \bibinfo{person}{Liwei Wang}, \bibinfo{person}{Chris~M. Cervantes}, \bibinfo{person}{Juan~C. Caicedo}, \bibinfo{person}{Julia Hockenmaier}, {and} \bibinfo{person}{Svetlana Lazebnik}.} \bibinfo{year}{2015}\natexlab{}.
\newblock \showarticletitle{Flickr30k Entities: Collecting Region-to-Phrase Correspondences for Richer Image-to-Sentence Models}. In \bibinfo{booktitle}{\emph{{ICCV}}}. \bibinfo{publisher}{{IEEE} Computer Society}, \bibinfo{pages}{2641--2649}.
\newblock


\bibitem[Tan et~al\mbox{.}(2023)]%
        {tan2023fedstar}
\bibfield{author}{\bibinfo{person}{Yue Tan}, \bibinfo{person}{Yixin Liu}, \bibinfo{person}{Guodong Long}, \bibinfo{person}{Jing Jiang}, \bibinfo{person}{Qinghua Lu}, {and} \bibinfo{person}{Chengqi Zhang}.} \bibinfo{year}{2023}\natexlab{}.
\newblock \showarticletitle{Federated Learning on Non-IID Graphs via Structural Knowledge Sharing}.
\newblock \bibinfo{journal}{\emph{Proceedings of the AAAI Conference on Artificial Intelligence}} \bibinfo{volume}{37}, \bibinfo{number}{8} (\bibinfo{year}{2023}), \bibinfo{pages}{9953--9961}.
\newblock
\href{https://doi.org/10.1609/AAAI.V37I8.26187}{doi:\nolinkurl{10.1609/AAAI.V37I8.26187}}


\bibitem[Tan et~al\mbox{.}(2022)]%
        {tan2022fedproto}
\bibfield{author}{\bibinfo{person}{Yue Tan}, \bibinfo{person}{Guodong Long}, \bibinfo{person}{Lu Liu}, \bibinfo{person}{Tianyi Zhou}, \bibinfo{person}{Qinghua Lu}, \bibinfo{person}{Jing Jiang}, {and} \bibinfo{person}{Chengqi Zhang}.} \bibinfo{year}{2022}\natexlab{}.
\newblock \showarticletitle{FedProto: Federated Prototype Learning across Heterogeneous Clients}.
\newblock \bibinfo{journal}{\emph{Proceedings of the AAAI Conference on Artificial Intelligence}} \bibinfo{volume}{36}, \bibinfo{number}{8} (\bibinfo{year}{2022}), \bibinfo{pages}{8432--8440}.
\newblock
\href{https://doi.org/10.1609/aaai.v36i8.20819}{doi:\nolinkurl{10.1609/aaai.v36i8.20819}}


\bibitem[Tang et~al\mbox{.}(2024)]%
        {Social_recomm}
\bibfield{author}{\bibinfo{person}{Mingwei Tang}, \bibinfo{person}{Meng Liu}, \bibinfo{person}{Hong Li}, \bibinfo{person}{Junjie Yang}, \bibinfo{person}{Chenglin Wei}, \bibinfo{person}{Boyang Li}, \bibinfo{person}{Dai Li}, \bibinfo{person}{Rengan Xu}, \bibinfo{person}{Yifan Xu}, \bibinfo{person}{Zehua Zhang}, \bibinfo{person}{Xiangyu Wang}, \bibinfo{person}{Linfeng Liu}, \bibinfo{person}{Yuelei Xie}, \bibinfo{person}{Chengye Liu}, \bibinfo{person}{Labib Fawaz}, \bibinfo{person}{Li Li}, \bibinfo{person}{Hongnan Wang}, \bibinfo{person}{Bill Zhu}, {and} \bibinfo{person}{Sri Reddy}.} \bibinfo{year}{2024}\natexlab{}.
\newblock \bibinfo{title}{Async Learned User Embeddings for Ads Delivery Optimization}.
\newblock
\showeprint[arxiv]{2406.05898}~[cs.IR]
\urldef\tempurl%
\url{https://arxiv.org/abs/2406.05898}
\showURL{%
\tempurl}


\bibitem[Tao et~al\mbox{.}(2020)]%
        {tao2020mgat}
\bibfield{author}{\bibinfo{person}{Zhulin Tao}, \bibinfo{person}{Yinwei Wei}, \bibinfo{person}{Xiang Wang}, \bibinfo{person}{Xiangnan He}, \bibinfo{person}{Xianglin Huang}, {and} \bibinfo{person}{Tat-Seng Chua}.} \bibinfo{year}{2020}\natexlab{}.
\newblock \showarticletitle{MGAT: Multimodal Graph Attention Network for Recommendation}.
\newblock \bibinfo{journal}{\emph{Information Processing \& Management}} \bibinfo{volume}{57}, \bibinfo{number}{5} (\bibinfo{year}{2020}), \bibinfo{pages}{102277}.
\newblock
\href{https://doi.org/10.1016/j.ipm.2020.102277}{doi:\nolinkurl{10.1016/j.ipm.2020.102277}}


\bibitem[Veli{\v{c}}kovi{\'c} et~al\mbox{.}(2018)]%
        {velickovic2018gat}
\bibfield{author}{\bibinfo{person}{Petar Veli{\v{c}}kovi{\'c}}, \bibinfo{person}{Guillem Cucurull}, \bibinfo{person}{Arantxa Casanova}, \bibinfo{person}{Adriana Romero}, \bibinfo{person}{Pietro Li{\`o}}, {and} \bibinfo{person}{Yoshua Bengio}.} \bibinfo{year}{2018}\natexlab{}.
\newblock \showarticletitle{Graph Attention Networks}. In \bibinfo{booktitle}{\emph{ICLR}}.
\newblock


\bibitem[Wan et~al\mbox{.}(2024)]%
        {wan2024fgl_fggp}
\bibfield{author}{\bibinfo{person}{Guancheng Wan}, \bibinfo{person}{Wenke Huang}, {and} \bibinfo{person}{Mang Ye}.} \bibinfo{year}{2024}\natexlab{}.
\newblock \showarticletitle{Federated Graph Learning under Domain Shift with Generalizable Prototypes}. In \bibinfo{booktitle}{\emph{Proceedings of the AAAI Conference on Artificial Intelligence}}, Vol.~\bibinfo{volume}{38}. \bibinfo{pages}{15429--15437}.
\newblock


\bibitem[Wang et~al\mbox{.}(2022)]%
        {WangFedScope_22_fsg}
\bibfield{author}{\bibinfo{person}{Zhen Wang}, \bibinfo{person}{Weirui Kuang}, \bibinfo{person}{Yuexiang Xie}, \bibinfo{person}{Liuyi Yao}, \bibinfo{person}{Yaliang Li}, \bibinfo{person}{Bolin Ding}, {and} \bibinfo{person}{Jingren Zhou}.} \bibinfo{year}{2022}\natexlab{}.
\newblock \showarticletitle{FederatedScope-GNN: Towards a Unified, Comprehensive and Efficient Package for Federated Graph Learning}. In \bibinfo{booktitle}{\emph{Proceedings of the 28th ACM SIGKDD Conference on Knowledge Discovery and Data Mining}}. \bibinfo{pages}{4110--4120}.
\newblock


\bibitem[Wei et~al\mbox{.}(2019)]%
        {wei2019mmgcn}
\bibfield{author}{\bibinfo{person}{Yinwei Wei}, \bibinfo{person}{Xiang Wang}, \bibinfo{person}{Liqiang Nie}, \bibinfo{person}{Xiangnan He}, \bibinfo{person}{Richang Hong}, {and} \bibinfo{person}{Tat-Seng Chua}.} \bibinfo{year}{2019}\natexlab{}.
\newblock \showarticletitle{MMGCN: Multi-modal Graph Convolution Network for Personalized Recommendation of Micro-video}. In \bibinfo{booktitle}{\emph{Proceedings of the 27th ACM International Conference on Multimedia}}. \bibinfo{publisher}{Association for Computing Machinery}, \bibinfo{address}{New York, NY, USA}, \bibinfo{pages}{1437--1445}.
\newblock
\href{https://doi.org/10.1145/3343031.3351034}{doi:\nolinkurl{10.1145/3343031.3351034}}


\bibitem[Wu et~al\mbox{.}(2023)]%
        {wufederated2023fedapp_gnn_bio2}
\bibfield{author}{\bibinfo{person}{Xiaotong Wu}, \bibinfo{person}{Jiaquan Gao}, \bibinfo{person}{Muhammad Bilal}, \bibinfo{person}{Fei Dai}, \bibinfo{person}{Xiaolong Xu}, \bibinfo{person}{Lianyong Qi}, {and} \bibinfo{person}{Wanchun Dou}.} \bibinfo{year}{2023}\natexlab{}.
\newblock \showarticletitle{Federated learning-based private medical knowledge graph for epidemic surveillance in internet of things}.
\newblock \bibinfo{journal}{\emph{Expert Systems}} (\bibinfo{year}{2023}), \bibinfo{pages}{e13372}.
\newblock


\bibitem[Wu et~al\mbox{.}(2020)]%
        {wu2020gnn_survey1}
\bibfield{author}{\bibinfo{person}{Zonghan Wu}, \bibinfo{person}{Shirui Pan}, \bibinfo{person}{Fengwen Chen}, \bibinfo{person}{Guodong Long}, \bibinfo{person}{Chengqi Zhang}, {and} \bibinfo{person}{Philip~S. Yu}.} \bibinfo{year}{2020}\natexlab{}.
\newblock \showarticletitle{A Comprehensive Survey on Graph Neural Networks}.
\newblock \bibinfo{journal}{\emph{IEEE Transactions on Neural Networks and Learning Systems}} \bibinfo{volume}{32}, \bibinfo{number}{1} (\bibinfo{year}{2020}), \bibinfo{pages}{4--24}.
\newblock


\bibitem[Xie et~al\mbox{.}(2021)]%
        {xie2021gcfl}
\bibfield{author}{\bibinfo{person}{Han Xie}, \bibinfo{person}{Jing Ma}, \bibinfo{person}{Li Xiong}, {and} \bibinfo{person}{Carl Yang}.} \bibinfo{year}{2021}\natexlab{}.
\newblock \bibinfo{title}{Federated Graph Classification over Non-IID Graphs}.
\newblock \bibinfo{howpublished}{arXiv preprint arXiv:2106.13423}.
\newblock
\href{https://doi.org/10.48550/arXiv.2106.13423}{doi:\nolinkurl{10.48550/arXiv.2106.13423}}


\bibitem[Xie et~al\mbox{.}(2024)]%
        {xie2024mhpflid}
\bibfield{author}{\bibinfo{person}{Liang Xie}, \bibinfo{person}{Ming Lin}, \bibinfo{person}{Tuan Luan}, \bibinfo{person}{Chao Li}, \bibinfo{person}{Yixuan Fang}, \bibinfo{person}{Qitao Shen}, {and} \bibinfo{person}{Zongwei Wu}.} \bibinfo{year}{2024}\natexlab{}.
\newblock \showarticletitle{MH-pFLID: Model Heterogeneous Personalized Federated Learning via Injection and Distillation for Medical Data Analysis}.
\newblock \bibinfo{journal}{\emph{arXiv preprint arXiv:2405.06822}} (\bibinfo{year}{2024}).
\newblock


\bibitem[Xu et~al\mbox{.}(2019)]%
        {xu2018gin}
\bibfield{author}{\bibinfo{person}{Keyulu Xu}, \bibinfo{person}{Weihua Hu}, \bibinfo{person}{Jure Leskovec}, {and} \bibinfo{person}{Stefanie Jegelka}.} \bibinfo{year}{2019}\natexlab{}.
\newblock \showarticletitle{How Powerful are Graph Neural Networks?}
\newblock \bibinfo{journal}{\emph{International Conference on Learning Representations}} (\bibinfo{year}{2019}).
\newblock


\bibitem[Yao et~al\mbox{.}(2024)]%
        {yao2024fgl_fedgcn}
\bibfield{author}{\bibinfo{person}{Yuhang Yao}, \bibinfo{person}{Weizhao Jin}, \bibinfo{person}{Srivatsan Ravi}, {and} \bibinfo{person}{Carlee Joe-Wong}.} \bibinfo{year}{2024}\natexlab{}.
\newblock \showarticletitle{FedGCN: Convergence-Communication Tradeoffs in Federated Training of Graph Convolutional Networks}.
\newblock \bibinfo{journal}{\emph{Advances in Neural Information Processing Systems}}  \bibinfo{volume}{36} (\bibinfo{year}{2024}).
\newblock


\bibitem[Zhang et~al\mbox{.}(2021a)]%
        {zhang2021fgl_survey_2}
\bibfield{author}{\bibinfo{person}{Huanding Zhang}, \bibinfo{person}{Tao Shen}, \bibinfo{person}{Fei Wu}, \bibinfo{person}{Mingyang Yin}, \bibinfo{person}{Hongxia Yang}, {and} \bibinfo{person}{Chao Wu}.} \bibinfo{year}{2021}\natexlab{a}.
\newblock \showarticletitle{Federated Graph Learning--A Position Paper}.
\newblock \bibinfo{journal}{\emph{arXiv preprint arXiv:2105.11099}} (\bibinfo{year}{2021}).
\newblock


\bibitem[Zhang et~al\mbox{.}(2025)]%
        {DY_bili_ku}
\bibfield{author}{\bibinfo{person}{Jiaqi Zhang}, \bibinfo{person}{Yu Cheng}, \bibinfo{person}{Yongxin Ni}, \bibinfo{person}{Yunzhu Pan}, \bibinfo{person}{Zheng Yuan}, \bibinfo{person}{Junchen Fu}, \bibinfo{person}{Youhua Li}, \bibinfo{person}{Jie Wang}, {and} \bibinfo{person}{Fajie Yuan}.} \bibinfo{year}{2025}\natexlab{}.
\newblock \showarticletitle{NineRec: {A} Benchmark Dataset Suite for Evaluating Transferable Recommendation}.
\newblock \bibinfo{journal}{\emph{{IEEE} Trans. Pattern Anal. Mach. Intell.}} \bibinfo{volume}{47}, \bibinfo{number}{7} (\bibinfo{year}{2025}), \bibinfo{pages}{5256--5267}.
\newblock


\bibitem[Zhang et~al\mbox{.}(2021b)]%
        {zhang2021fedsage}
\bibfield{author}{\bibinfo{person}{Ke Zhang}, \bibinfo{person}{Carl Yang}, \bibinfo{person}{Xiaoxiao Li}, \bibinfo{person}{Lichao Sun}, {and} \bibinfo{person}{Siu-Ming Yiu}.} \bibinfo{year}{2021}\natexlab{b}.
\newblock \showarticletitle{Subgraph Federated Learning with Missing Neighbor Generation}. In \bibinfo{booktitle}{\emph{Advances in Neural Information Processing Systems}}, Vol.~\bibinfo{volume}{34}. \bibinfo{publisher}{Curran Associates, Inc.}, \bibinfo{address}{Red Hook, NY, USA}, \bibinfo{pages}{6671--6682}.
\newblock
\urldef\tempurl%
\url{https://proceedings.neurips.cc/paper/2021/hash/34adeb8e3242824038aa65460a47c29e-Abstract.html}
\showURL{%
\tempurl}


\bibitem[Zhang and Chen(2018)]%
        {Zhang18link_prediction1}
\bibfield{author}{\bibinfo{person}{Muhan Zhang} {and} \bibinfo{person}{Yixin Chen}.} \bibinfo{year}{2018}\natexlab{}.
\newblock \showarticletitle{Link Prediction Based on Graph Neural Networks}.
\newblock \bibinfo{journal}{\emph{Advances in Neural Information Processing Systems}} (\bibinfo{year}{2018}).
\newblock


\bibitem[Zhou et~al\mbox{.}(2022)]%
        {zhou2022gnn_survey2}
\bibfield{author}{\bibinfo{person}{Yu Zhou}, \bibinfo{person}{Haixia Zheng}, \bibinfo{person}{Xin Huang}, \bibinfo{person}{Shufeng Hao}, \bibinfo{person}{Dengao Li}, {and} \bibinfo{person}{Jumin Zhao}.} \bibinfo{year}{2022}\natexlab{}.
\newblock \showarticletitle{Graph Neural Networks: Taxonomy, Advances, and Trends}.
\newblock \bibinfo{journal}{\emph{ACM Transactions on Intelligent Systems and Technology}} \bibinfo{volume}{13}, \bibinfo{number}{1} (\bibinfo{year}{2022}), \bibinfo{pages}{1--54}.
\newblock


\bibitem[Zhu et~al\mbox{.}(2025)]%
        {zhu2025mmgraph}
\bibfield{author}{\bibinfo{person}{Jing Zhu}, \bibinfo{person}{Yuhang Zhou}, \bibinfo{person}{Shengyi Qian}, \bibinfo{person}{Zhongmou He}, \bibinfo{person}{Tong Zhao}, \bibinfo{person}{Neil Shah}, {and} \bibinfo{person}{Danai Koutra}.} \bibinfo{year}{2025}\natexlab{}.
\newblock \showarticletitle{Mosaic of Modalities: A Comprehensive Benchmark for Multimodal Graph Learning}. In \bibinfo{booktitle}{\emph{Proceedings of the IEEE/CVF Conference on Computer Vision and Pattern Recognition}}. \bibinfo{publisher}{IEEE}, \bibinfo{address}{Piscataway, NJ, USA}, \bibinfo{pages}{14215--14224}.
\newblock
\href{https://doi.org/10.1109/CVPR52734.2025.01326}{doi:\nolinkurl{10.1109/CVPR52734.2025.01326}}


\bibitem[Zhu et~al\mbox{.}(2024)]%
        {zhu2024fedtad}
\bibfield{author}{\bibinfo{person}{Yinlin Zhu}, \bibinfo{person}{Xunkai Li}, \bibinfo{person}{Zhengyu Wu}, \bibinfo{person}{Di Wu}, \bibinfo{person}{Miao Hu}, {and} \bibinfo{person}{Rong-Hua Li}.} \bibinfo{year}{2024}\natexlab{}.
\newblock \showarticletitle{FedTAD: Topology-aware Data-free Knowledge Distillation for Subgraph Federated Learning}.
\newblock \bibinfo{journal}{\emph{arXiv preprint arXiv:2404.14061}} (\bibinfo{year}{2024}).
\newblock


\end{thebibliography}

\end{document}